\documentclass{article} 
\usepackage{iclr2023_conference,times}

\usepackage{algorithm}
\usepackage{algorithmic}
\usepackage{amsmath,mathtools,amsthm,dsfont,amsfonts}
\usepackage[mathscr]{euscript}
\usepackage{amssymb}
\usepackage{enumitem}
\usepackage[breaklinks=true,colorlinks,citecolor=blue,linkcolor=blue]{hyperref}
\usepackage{url}
\usepackage{multirow}
\usepackage{multicol}
\usepackage{thmtools}
\usepackage{wrapfig}

\theoremstyle{definition}
\newtheorem{theorem}{Theorem}[section]
\newtheorem{lemma}[theorem]{Lemma}

\newtheorem{definition}{Definition}[section]

\newcommand{\bc}[1]{\left\{{#1}\right\}}
\newcommand{\br}[1]{\left({#1}\right)}
\newcommand{\bs}[1]{\left[{#1}\right]}

\newcommand{\abs}[1]{\left| {#1} \right|}

\newcommand{\x}{{\textrm x}}

\renewcommand{\P}{\textrm{P}}

\newcommand{\bb}{\mathbb}

\newcommand{\R}{\bb R}

\newcommand{\E}{\bb E}

\newcommand{\cB}{\mathcal B}

\newcommand{\cD}{\mathcal D}
\newcommand{\cF}{\mathcal F}
\newcommand{\cO}{\mathcal O}
\newcommand{\cS}{\mathcal S}
\newcommand{\cX}{\mathcal X}
\newcommand{\cY}{\mathcal Y}
\newcommand{\cZ}{\mathcal Z}

\newcommand{\cH}{\mathcal H}
\newcommand{\cP}{\mathcal P}

\DeclareMathOperator*{\argmax}{argmax}
\DeclareMathOperator*{\argmin}{argmin}

\newcommand{\IWALSone}{\textrm{S}}

\newcommand{\IWEBS}{\textrm{IWeS}}
\newcommand{\IWEBSV}{\textrm{IWeS-V}}
\newcommand{\AL}{\textrm{AL}}
\newcommand{\IWAL}{\textrm{IWAL}}

\usepackage{hyperref}
\usepackage{url}

\title{Leveraging Importance Weights in Subset Selection}

\author{Gui Citovsky$^1$, Giulia DeSalvo$^1$, Sanjiv Kumar$^1$, Srikumar Ramalingam$^1$,\\
\textbf{Afshin Rostamizadeh$^1$, Yunjuan Wang$^2$\thanks{This work was done when the author was interning at Google.}} \\
$^1$Google Research, New York, NY 10011\\
$^2$Department of Computer Science, Johns Hopkins University, Baltimore, MD, 21218\\
\texttt{\{gcitovsky,giuliad,sanjivk,rostami,rsrikumar\}@google.com} \\
\texttt{ywang509@jhu.edu}
}

\iclrfinalcopy 
\begin{document}

\maketitle

\begin{abstract}

We present a subset selection algorithm designed to work with arbitrary model families in a practical batch setting. In such a setting, an algorithm can sample examples one at a time but, in order to limit overhead costs, is only able to update its state (i.e. further train model weights) once a large enough batch of examples is selected.  Our algorithm, IWeS, selects examples by importance sampling where the sampling probability assigned to each example is based on the entropy of models trained on previously selected batches. IWeS admits significant performance improvement compared to other subset selection algorithms for seven publicly available datasets. Additionally, it is competitive in an active learning setting, where the label information is not available at selection time.
We also provide an initial theoretical analysis to support our importance weighting approach, proving generalization and sampling rate bounds.
 
\end{abstract}

\section{Introduction}

Deep neural networks have shown remarkable success in several domains such as computer vision and natural language processing. In many tasks, this is achieved by heavily relying on extremely large labeled datasets. In addition to the storage costs and potential security/privacy concerns that come along with large datasets, training modern deep neural networks on such datasets also incur high computational costs.
With the growing size of datasets in various domains, algorithm scalability is a real and imminent challenge that needs to be addressed.
One promising way to solve this problem is with data subset selection, where the learner aims to find the most informative subset from a large number of training samples to approximate (or even improve upon) training with the entire training set. Such ideas have been extensively studied in k-means and k-median clustering \citep{har2004coresets}, subspace approximation \citep{feldman2010coresets}, computational geometry \citep{agarwal2005geometric}, density estimation \citep{turner2021statistical}, to name a few.

One particular approach for solving data subsampling involves the computation of coresets, which are weighted subsets of a dataset that can act as the proxy for the whole dataset to solve some optimization task. Coreset algorithms are primarily motivated with theoretical guarantees that bound the difference between the training loss (or other such objective) over the coreset and that over the full dataset under different assumptions on the losses and hypothesis classes~\citep{mai2021coresets,Munteanu_NeurIPS2018,Curtin_2019,karninedo2019}. However, in practice, most competitive subset selection algorithms, that are designed for general loss functions and arbitrary function classes, focus only on selecting informative subsets of the data and typically do not assign weights to the selected examples. These methods are, for example, based on some notion of model uncertainty \citep{scheffer2001active}, information gain \citep{argamon1999committee}, loss gradients \citep{paul2021deep,ash2019deep}, or diversity \citep{sener2017active}. Counter to this trend, we show that weighting the selected samples can be very beneficial.

In this work, we present a subset selection algorithm called \IWEBS\ that is designed for general loss functions and hypothesis classes and that selects examples by importance sampling, a theoretically motivated and unbiased sampling technique. Importance sampling is conducted according to a specially crafted probability distribution and, importantly, each sampled example is weighted inversely proportional to its sampling probability when computing the training loss. We develop two types of sampling probability for different practical requirements (e.g. computational constraints and label availability), but in both cases, the sampling probability is based on the example's entropy-based score computed using a previously trained model.
We note, the \IWEBS\ algorithm is similar to the IWAL active learning algorithm of \cite{beygelzimer2009importance} as both are based on importance sampling. However, in contrast to IWAL, \IWEBS\ uses a different sampling probability definition with a focus on providing a practical method that is amenable to large deep networks and complex hypothesis classes.

Through extensive experiments, we find that the \IWEBS\ algorithm is competitive for deep neural networks over several datasets.
We compare our algorithm against four types of baselines whose sampling strategies leverage: the model's uncertainty over examples, diversity of selected examples, gradient information, and random sampling.
Finally, we analyze a closely related albeit less practical algorithm that inspires the design of \IWEBS, called \IWEBSV , proving it admits generalization and sampling rate guarantees that hold for general loss functions and hypothesis classes.

The contributions of this work can be summarized as follows:
\vspace{-0.5em}
\begin{enumerate}[label=\arabic*.,leftmargin=*]
\item We present the \textbf{I}mportance \textbf{We}ighted \textbf{S}ubset Selection (\IWEBS) algorithm that selects examples by importance sampling with a sampling probability based on a model's entropy, which is applicable to (and practical for) arbitrary model families  including modern deep networks. In addition to the subset selection framework, \IWEBS\ also works in the active learning setting where the examples are unlabeled at selection time.

\item We demonstrate that \IWEBS\ achieves significant improvement over several baselines (Random, Margin, Least-Confident, Entropy, Coreset, BADGE) using VGG16 model for six common multi-class datasets (CIFAR10, CIFAR10-corrupted, CIFAR100, SVHN, Eurosat, Fashion MNIST), and using ResNet101 model for the large-scale multi-label OpenImages dataset.

\item We provide a theoretical analysis for a closely related algorithm, \IWEBSV, in Section \ref{sec:theory}. We prove a $\cO(1/\sqrt{T})$ generalization bound, which depends on the full training dataset size $T$.  We further give a new definition of disagreement coefficient and prove a sampling rate bound by leveraging label information, which is tighter compared with the label complexity bound provided by \cite{beygelzimer2009importance} that does not use label information.
\end{enumerate}
\vspace{-1em}

\subsection{Related Work}
\label{sec:related}
\vspace{-0.5em}

\textbf{Uncertainty.}
Uncertainty sampling, which selects examples that the model is least confident on, is favored by practitioners \citep{mussmann2018uncertainty} and rather competitive among many recent algorithms \citep{yang2018benchmark}. Uncertainty can be measured through entropy \citep{argamon1999committee}, least confidence \citep{culotta2005reducing}, and most popular is the margin between the most likely and the second most likely labels \citep{scheffer2001active}.
\cite{beygelzimer2009importance} makes use of a disagreement-based notion of uncertainty and constructs an importance weighted predictor with theoretical guarantees called IWAL, which is further enhanced by \cite{cortesdesalvogentilmohrizhang2019}. However, IWAL is not directly suitable for use with complex hypothesis spaces, such as deep networks, since it requires solving a non-trivial optimization over a subset of the hypothesis class, the so-called version space, in order to compute sampling probabilities.
We further discuss these difficulties in Section \ref{sec:theory}.

\textbf{Diversity.}
In another line of research, subsets are selected by enforcing diversity such as in the FASS \citep{wei2015submodularity} and  Coreset  \citep{sener2017active} algorithms. \cite{wei2015submodularity} introduces a submodular sampling objective that trades off between uncertainty and diversity by finding a diverse set of samples from amongst those that the current trained model is most uncertain about. It was further explored by \cite{kaushal2019learning} who designed a unified framework for data subset selection  with facility location and dispersion-based diversity functions. \cite{sener2017active} show that the task of identifying a coreset in an active learning setting can be mapped to solving the k-center problem.
Further recent works related to coreset idea are \cite{mirzasoleiman2020coresets,killamsetty2021grad}, where the algorithms select representative subsets of the training data to minimize the estimation error between the weighted gradient of selected subset and the full gradient.

\textbf{Loss Gradient.}
Another class of algorithms selects a subset by leveraging the loss gradients. For example,
the GRAND score \citep{paul2021deep}, or closely related EL2N score, leverages the average gradient across several different independent models to measure the importance of each sample.
However, as such, it requires training several neural networks, which is computationally expensive.
BADGE \citep{ash2019deep} is a sampling strategy for deep neural networks that uses k-MEANS++ on the gradient embedding of the networks to balance between uncertainty and diversity.
Finally, for the sake of completeness, we note that importance weighting type approaches have also been used for the selection of examples within an SGD minibatch  \citep{katharopoulos2018not,johnson2018training}, which can be thought of a change to the training procedure itself. In contrast, the problem setting we consider in this work requires explicitly producing a (weighted) subset of the training data and treats the training procedure itself as a black-box.

These are a sampling of data subset selection algorithms, and we refer the reader to \citep{guo2022deepcore} for a more detailed survey. In this work, we choose at least one algorithm from each of the categories mentioned above, in particular, Margin \citep{scheffer2001active},  BADGE \citep{ash2019deep}, and Coreset \citep{sener2017active} to compare against empirically in Section~\ref{sec:exp}. However, before that, we first formally define the \IWEBS\ algorithm in the following section.
\looseness=-1

\section{The \IWEBS\ Algorithm}\label{sec:algo}
\begin{algorithm}[t]
\caption{\textbf{I}mportance \textbf{We}ighted \textbf{S}ubset Selection (\IWEBS)}
\centering
\begin{algorithmic}[1]\label{algo:iwals}
\REQUIRE Labeled pool $\mathcal{P}$, seed set size $k_0$, subset batch size $k$, number of iterations $R$, weight cap parameter $u$.
\STATE Initialize the subset $\cS=\emptyset$.
\STATE $\cS_{0} \leftarrow$ Draw $k_0$ examples from $\mathcal{P}$ uniformly at random.
\STATE Set $\cS = \{(\x,y,1):(\x,y) \in \cS_0\}$ and $\mathcal{P}=\mathcal{P} \backslash \cS_0$
\FOR{$r=1,2,\ldots,R$}
\STATE Set $\cS_r=\emptyset$.
\STATE Train $f_r, g_r$ on $\cS$ using the weighted loss and independent random initializations.
\WHILE{$|\cS_r|<k$}
\STATE Select $(\x,y)$ uniformly at random from $\mathcal{P}$.
\STATE Set $p(\x,y)$ using entropy-based disagreement or entropy criteria shown in Eq~\eqref{eq:disagreement_strategy} and ~\eqref{eq:entropy_strategy}.
\STATE $Q\sim\textrm{Bernoulli}(p(\x,y))$.
\IF{$Q=1$}
\STATE Set $\cS_r=\cS_r\cup\bc{\br{\x,y,\min\br{\frac{1}{p(\x,y)}, u}}}$ and $\mathcal{P} =\mathcal{P} \backslash\bc{\br{\x,y}}$.
\ENDIF
\ENDWHILE
\STATE $\cS = \cS \cup \cS_r$.
\ENDFOR
\STATE  Train $f_{R+1}$ on $\cS$ using the weighted loss.
\RETURN $\cS,f_{R+1}$
\end{algorithmic}
\end{algorithm}

We consider a practical batch streaming setting, where an algorithm processes one example at a time without updating its state until a batch of examples is selected. That is, like in standard streaming settings, the algorithm receives a labeled example, and decides whether to include it in the selected subset or not. Yet, the algorithm is only allowed to update its state after a fixed batch of examples have been selected in order to limit the overhead costs (e.g. this typically can include retraining models and extracting gradients).
Unlike the pool-based setting where the algorithm receives the entire labeled pool beforehand, a batch streaming setting can be more appropriate when facing a vast training data pool since the algorithm can only process a subset of the pool without iterating over the whole pool. Note that any batch streaming algorithm can also be used in a pool-based setting, by simply streaming through the pool in a uniformly random fashion.
At a high level, the \IWEBS\ algorithm selects examples by importance sampling where the sampling probability is based on the entropy of models trained on previously selected data. We define two sampling probabilities that allow us to trade-off between performance and the computational cost, as well as label-aware or an active learning setting leading to less label annotation costs. As we will subsequently see, these sampling definitions are both easy to use and work well in practice.

To define the algorithm in more detail, we let $\cX\in\R^d$ and $\cY = \{1, \ldots, c\}$ denote the input space and the multi-class label space, respectively. We assume the data $(\x,y)$ is drawn from an unknown joint distribution $\cD$ on $\cX\times\cY$.
 Let $\cH=\{h:\cX\rightarrow \cZ\}$ be the hypothesis class consisting of functions mapping from $\cX$ to some prediction space $\cZ\subset \R ^\cY$ and let $\ell:\cZ\times\cY\rightarrow\R$ denote the loss.
\looseness=-1

The pseudocode of \IWEBS\ is shown in Algorithm \ref{algo:iwals}. Initially, a seed set $\cS_0$ ($|\cS_0|=k_0$) is selected uniformly at random from the labeled pool $\cP$.
Then the algorithm proceeds in rounds $r\in [1,\ldots, R]$ and it consists of two main components: training and sampling.
At the training step at round $r$, the model(s) are trained using the importance-weighted loss, namely $f_{r}=\arg\min_{h\in\cH}\sum_{(\x,y,w)\in\cS}w\cdot\ell(h(\x),y)$ on the subset $\cS$, selected so far, in the previous $r-1$ rounds. Depending on the sampling strategy, we may need to randomly initialize two models $f_{r}, g_{r}$ in which case they are trained independently on the same selected subset $\cS$, but with different random initializations.
At the sampling step at round $r$, the \IWEBS\ algorithm calculates a sampling probably for example $(\x,y) \in \cS$  based on one of the following  definitions:
\vspace{-0.5em}
\begin{itemize}[leftmargin=*]
\item {\bf Entropy-based Disagreement.} We define the sampling probability based on the disagreement on two functions with respect to entropy restricted to the labeled example $(\x, y)$. That is,
\begin{align}\label{eq:disagreement_strategy}
p(\x,y)=\abs{\P_{f_{r}}(y|\x)\log{\P_{f_{r}}(y|\x)}-\P_{g_{r}}(y|\x)\log{\P_{g_{r}}(y|\x)}}
\end{align}
where $\P_{f_{r}}(y|\x)$ is the probability of class $y$ with model $f_{r}$ given example $\x$. If the two functions, $f_{r}, g_{r}$, disagree on the labeled example $(\x, y)$, then $p(\x,y)$ will be small and the example will be less likely to be selected.  This definition is the closest to the \IWEBSV\ algorithm analyzed in Section~\ref{sec:theory} and achieves the best performance when the computational cost of training two models is not an issue. In Appendix~\ref{sec:extexp}, we show an efficient version of entropy-based disagreement that utilizes only one model and achieves similar performance.

\item {\bf Entropy.} We define the sampling probability by the normalized entropy of the model $f_{r}$ trained on past selected examples:
\begin{align}\label{eq:entropy_strategy}
p(\x,\cdot)=-\sum_{y'\in\cY}\P_{f_{r}}(y'|\x)\log_2{\P_{f_{r}}(y'|\x)}/\log_2{|\cY|}.
\end{align}
The sampling probability $p(\x, \cdot)$ is high whenever the model class probability $\P_{f_{r}}(y'|\x)$ is close to $1/|\cY|$, which is when the model is not confident about its prediction as it effectively randomly selects a label from $\cY$.   This definition does not use the label $y$ and thus it can be used in an active learning setting where the algorithm can only access the unlabeled examples.  Another advantage is that it only requires training one model, thereby saving some computational cost.
\end{itemize}
\vspace{-0.5em}
We note that entropy-based sampling has been used in algorithms such as uncertainty sampling as discussed in the related works section, but using entropy to define importance weights has not been done in past literature.

Based on one of these definitions, the \IWEBS\ algorithm then decides whether to include the example into the selected subset $\cS$ by flipping a coin $Q$ with chosen sampling probability $p(\x,y)$. If the example is selected, the example's corresponding weight $w$ is set to $\frac{1}{p(\x,y)}$, and the example is removed from the labeled pool $\cP=\cP\backslash\bc{\br{\x,y}}$.  This process is repeated until $k$ examples have been selected.  Below we use \IWEBS-dis as an abbreviation for \IWEBS\ algorithm with Entropy-based Disagreement sampling probability and \IWEBS-ent for \IWEBS\ algorithm with Entropy sampling probability.
\looseness=-1

The weighted loss used to train the model can be written as $\frac{1}{|\cP|}\sum_{i\in\cP}\frac{Q_i}{p(\x_i,y_i)}\ell(f(\x_i),y_i)$ and it is an unbiased estimator of the population risk $\E_{(\x,y)\sim\cD}[\ell(f(\x),y)]$. Yet such estimator can have a large variance when the model is highly confident in its prediction, that is whenever $\P_{f_{r}}(y|\x)$ is large, then $p(\x, y)$ is small. This may lead to training instability and one pragmatic approach to addressing this issue is by ``clipping'' the importance sampling weights \citep{ionides2008truncated,swaminathan2015counterfactual}. Thus in our algorithm, we let $u$ be the upper bound on the weight of the selected example. Although this clipping strategy introduces an additional parameter, we find it is not too sensitive and, as mentioned in the empirical section, set it to a fixed constant throughout our evaluation.

\section{Empirical Evaluation}\label{sec:exp}
We compare \IWEBS\ with state-of-the-art baselines on several image classification benchmarks. Specifically, we consider six multi-class datasets (CIFAR10, CIFAR100 \citep{krizhevsky2009learning}, SVHN \citep{netzer2011reading}, EUROSAT \citep{helber2019eurosat}, CIFAR10 Corrupted \citep{hendrycks2018benchmarking}, Fashion MNIST \citep{xiao2017fashion} and one large-scale multi-label Open Images dataset~\citep{krasin2017openimages}.
In the multi-class setting, each image is associated with only one label. On the other hand, the multi-label Open Images dataset consists of 19,957 classes over 9M images, where each image contains binary labels for a small subset of the classes (on average 6 labels per image).
Further details of each dataset can be found in Table~\ref{tab:datasets} and Table~\ref{tab:openimages} in the appendix.

For all experiments, we consider a diverse set of standard baselines from both subset selection and active learning literature (discussed in Section~\ref{sec:related}).
\vspace{-0.5em}
\begin{itemize}[leftmargin=*]
\item \textbf{Uncertainty Sampling} selects top $k$ examples on which the current model admits the highest uncertainty.  There are three popular ways of defining model uncertainty $s(\x)$ of an example $\x$, namely margin sampling, entropy sampling, and least confident sampling, and all are based on $\P_f[\hat y|\x]$, the probability of class $\hat y$ given example $\x$ according to the model $f$. Margin sampling defines the model uncertainty of an example $\x$ as $s(\x)=1-(\P_f[\hat y_1|\x]-\P_f[\hat y_2|\x])$ where  $\hat y_1=\argmax_{y\in\cY}\P_f[y|\x], \hat y_2=\argmax_{y\in\cY\backslash y_1}\P_f[y|\x]$ are the first and second most probable classes for model $f$. For entropy sampling, model uncertainty is defined as $s(\x)=-\sum_{y\in\cY}\P_f(\hat y|\x)\log(\P_f(\hat y|\x))$ while for least confidence sampling, it is defined as $s(\x)=1-\max_{y\in\cY}\P_f(\hat y|\x)$.
\item \textbf{BADGE} of
\citet{ash2019deep} selects $k$ examples by using the $k$-MEANS++ seeding algorithm using the gradient vectors, computed with respect to the penultimate layer using the most likely labels given by the latest model checkpoint.
\item  \textbf{Coreset ($k$-Center)}  of \citet{sener2017active} selects a subset of examples using their embeddings derived from the penultimate layer using the latest model checkpoint. In particular, the $k$ examples are chosen using a greedy 2-approximation algorithm for the $k$-center problem.
\item \textbf{Random Sampling} selects $k$ examples uniformly at random.
\end{itemize}
\vspace{-0.5em}

\subsection{Multi-class Experiments}
\vspace{-0.5em}

 \setlength{\tabcolsep}{0.5pt}
\begin{figure}[t]
\centering
\begin{tabular}{ccc}
\includegraphics[width=0.32\textwidth]{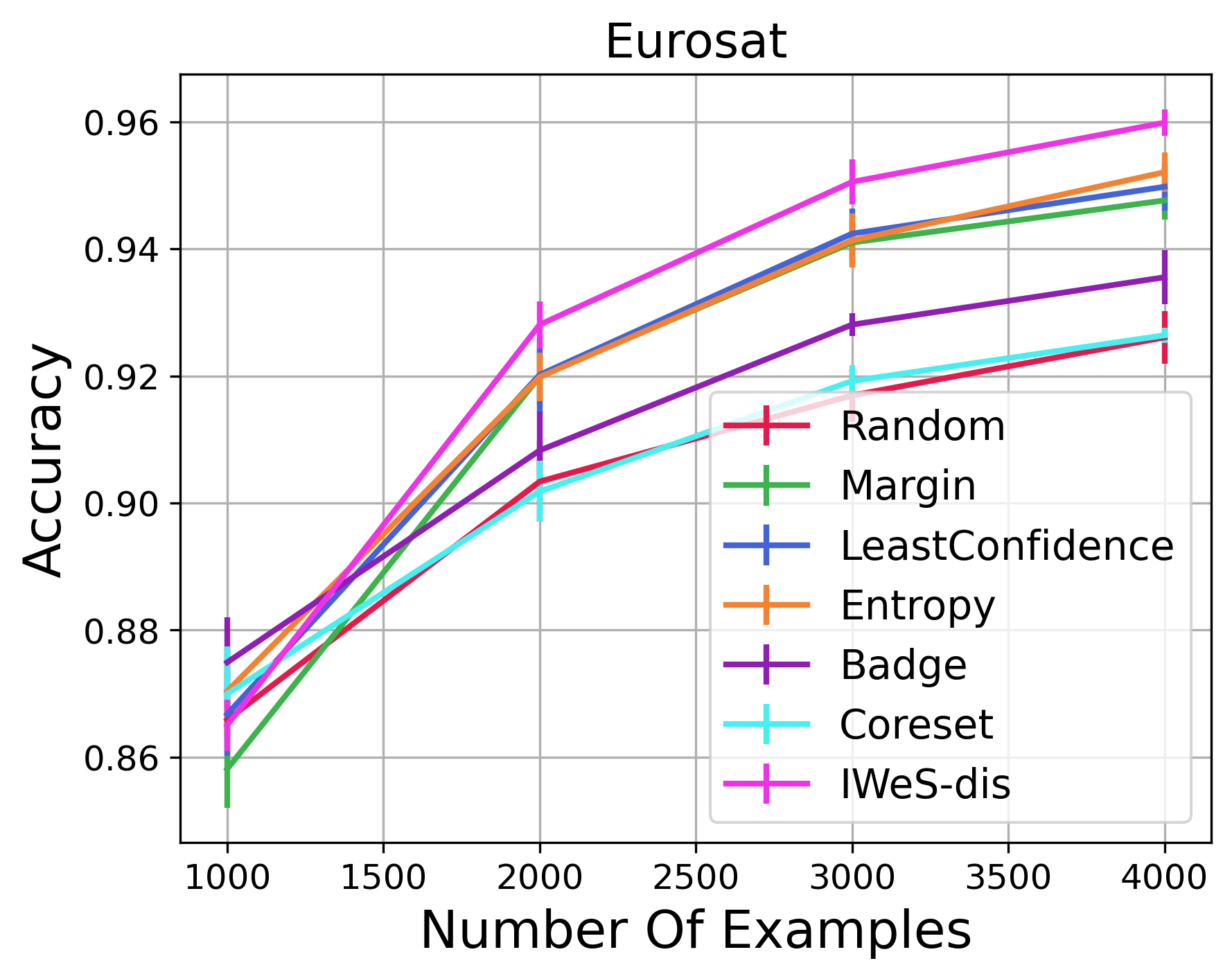}
&
\includegraphics[width=0.32\textwidth]{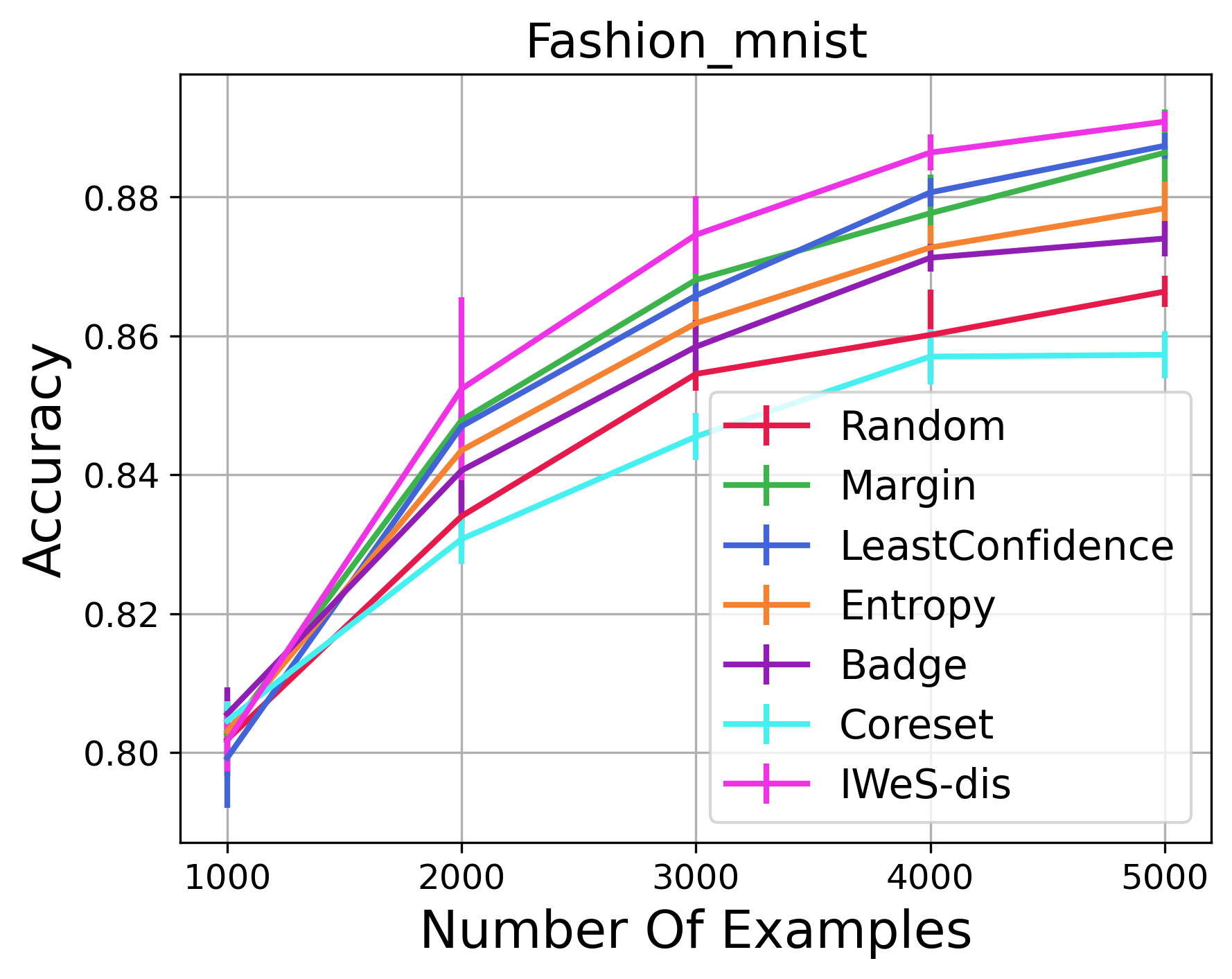}
&\includegraphics[width=0.32\textwidth]{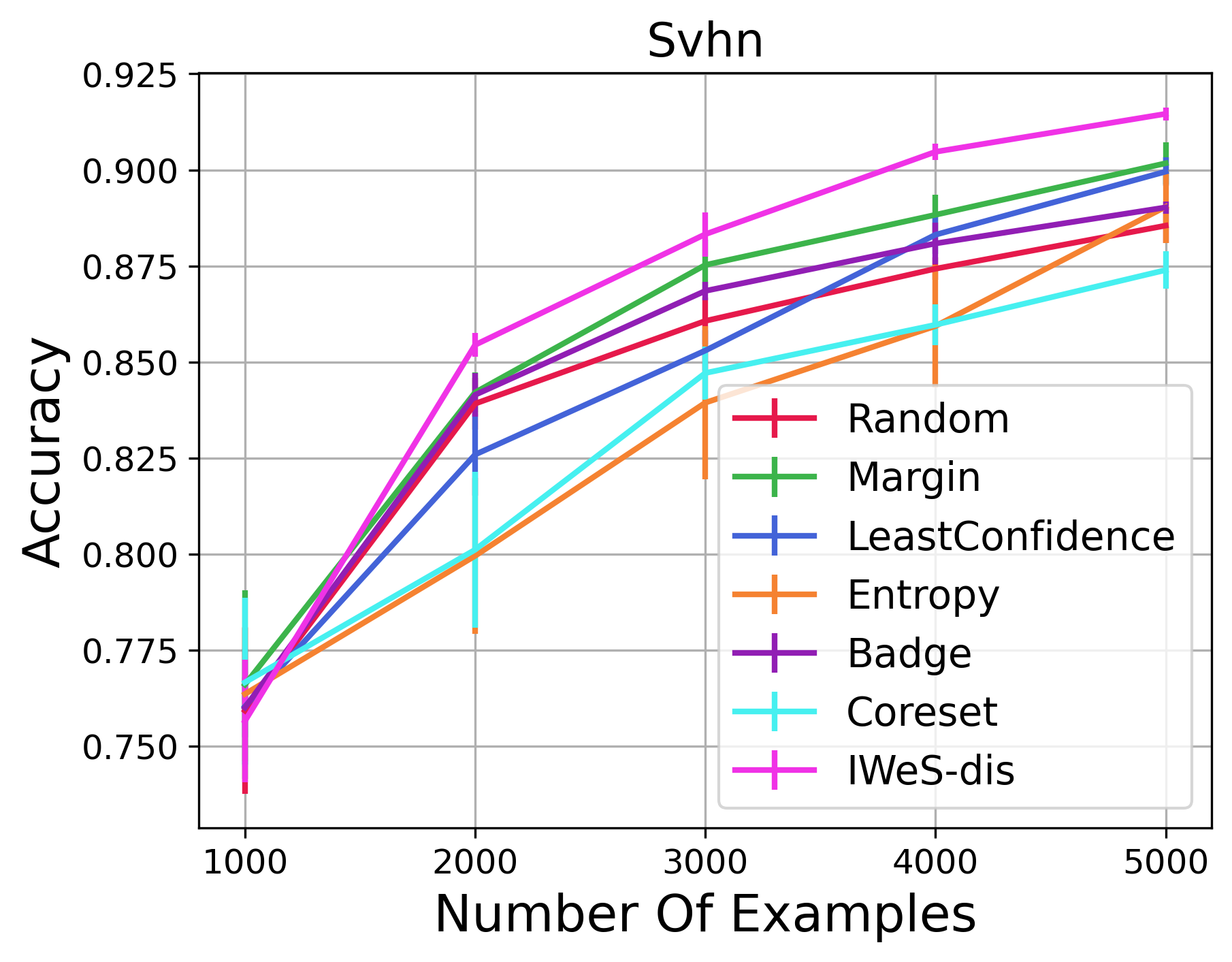}\\
\includegraphics[width=0.32\textwidth]{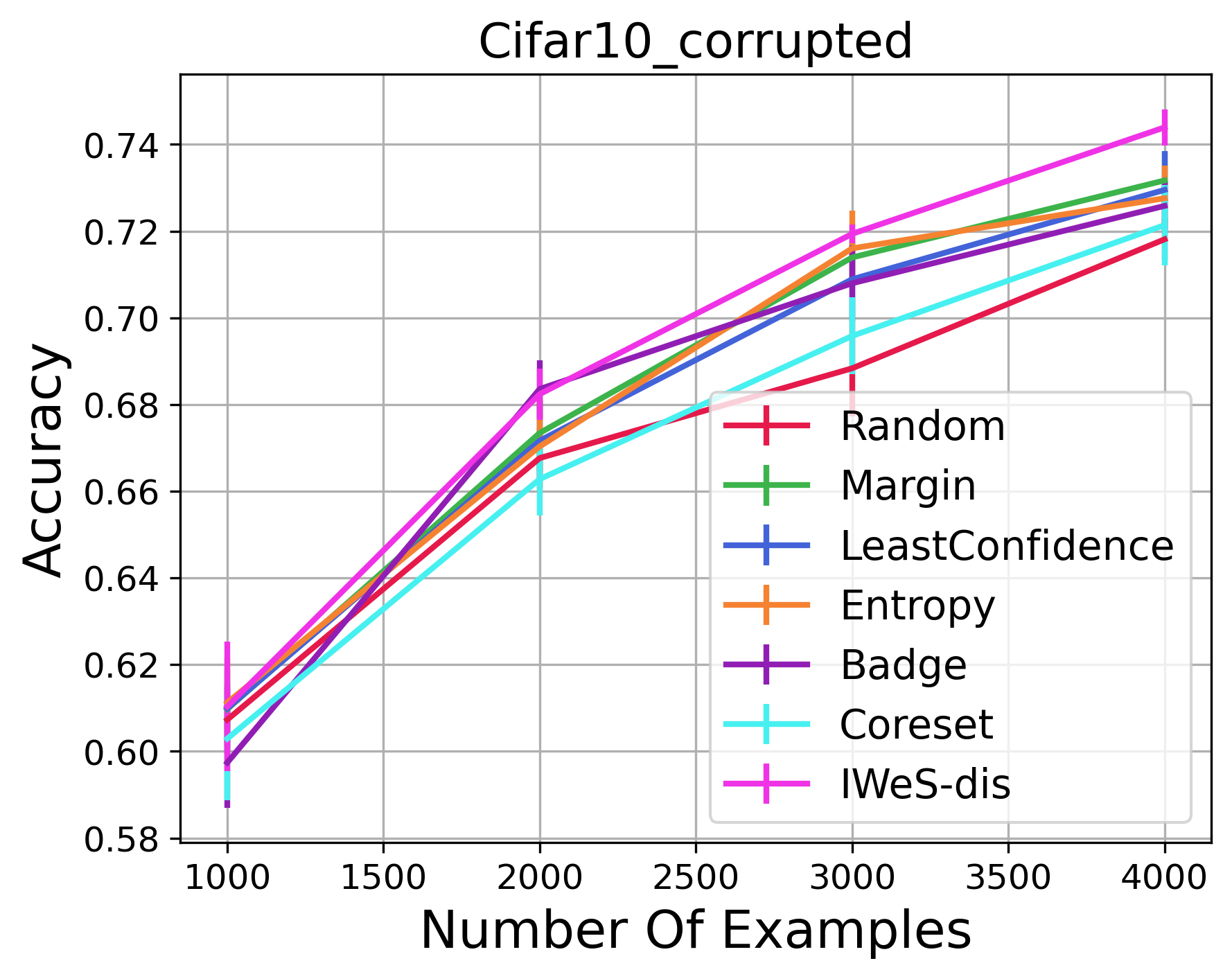}
&
\includegraphics[width=0.32\textwidth]{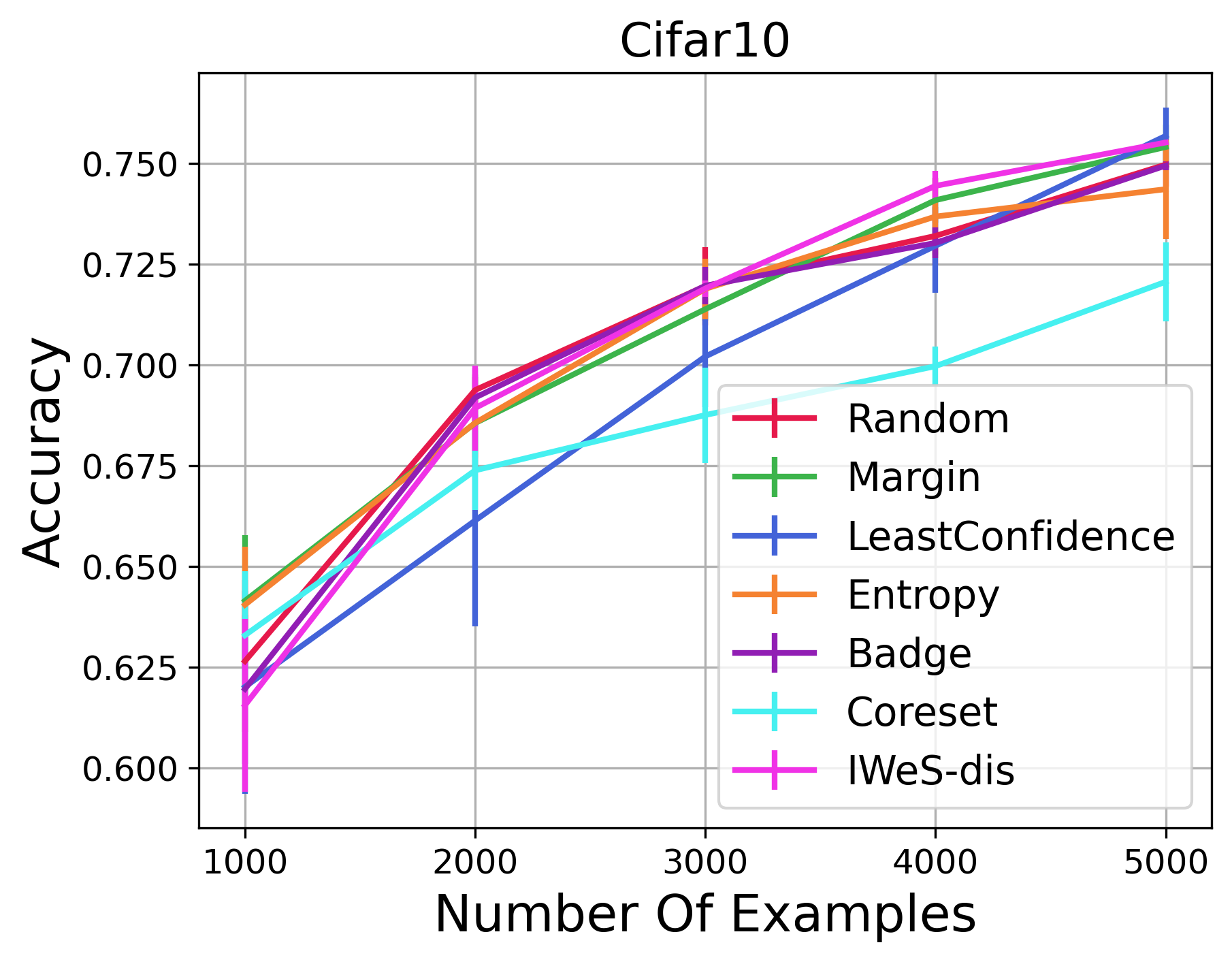}
&
\includegraphics[width=0.32\textwidth]{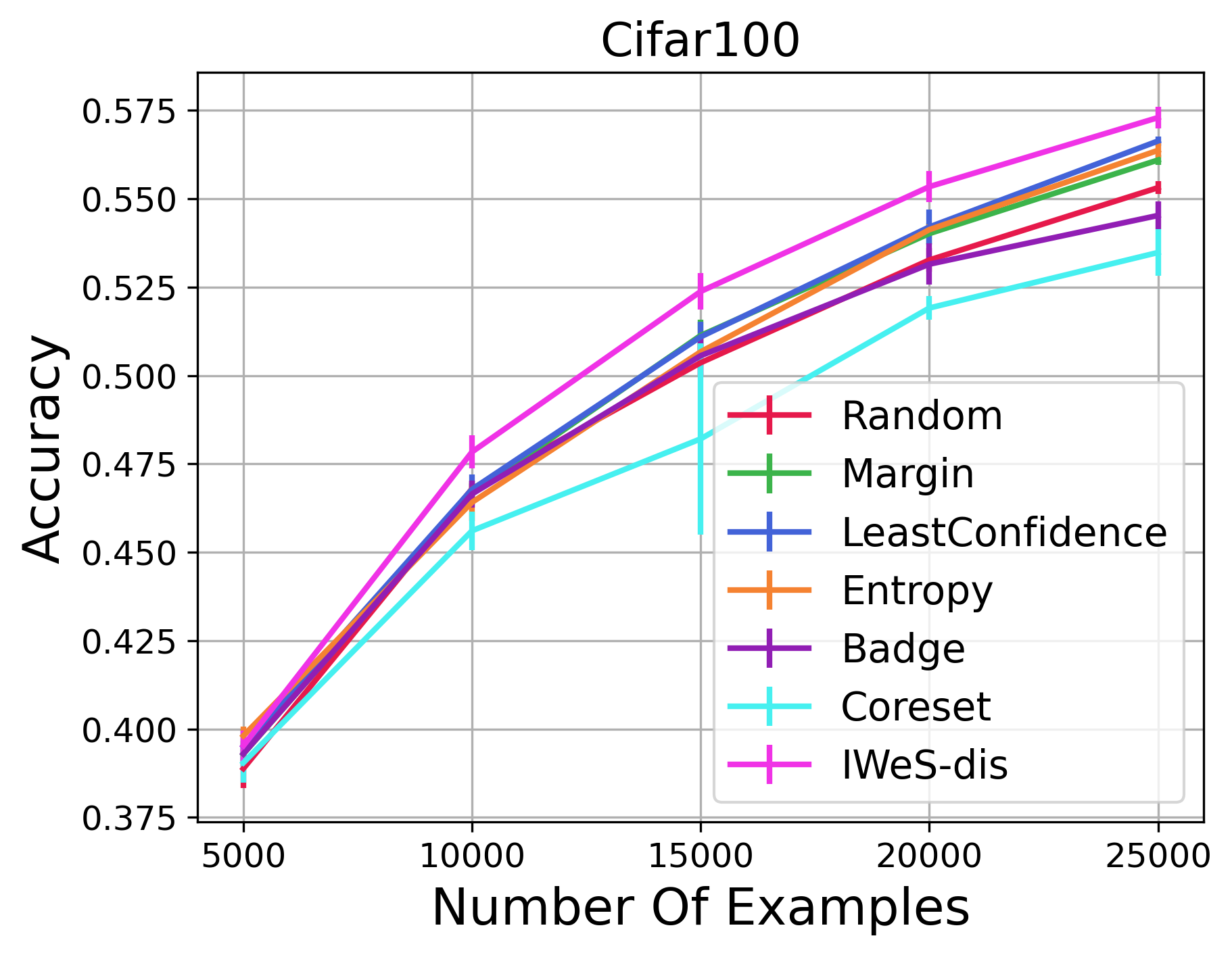}
\end{tabular}
\caption{Accuracy of VGG16 when trained on examples selected by  \IWEBS-dis and baseline algorithms.
}
\vspace{-10pt}
\label{fig:vgg16}
\end{figure}

Here, we compare the \IWEBS\ algorithm against the baselines on the six multi-class image datasets. We use the VGG16 architecture with weights that were pre-trained using ImageNet as well as add two fully-connected 4096 dimensional layers and a final prediction layer. Xavier uniform initialization is used for the final layers.
For each dataset, we tune the learning rate by choosing the rate from the set $\{0.001, 0.002, 0.005, 0.01, 0.1 \}$ that achieves best model performance on the seed set.  We use batch SGD with the selected learning rate and fix SGD’s batch size to 100. At each sampling round $r$, the model is trained to convergence on all past selected examples.
 For \IWEBS , we set the weight capping parameter to 2 for all datasets except for CIFAR10 which we decreased to 1.5 in order to reduce training instability.

The embedding layer for BADGE and Coreset is extracted from the penultimate layer having a dimension of 4096. The effective dimension of the gradient vector in BADGE grows with the number of labels, which is problematic for CIFAR100 as it has 100 classes. More specifically, the runtime of BADGE is given by $\cO\br{dkT}$, which can be large for CIFAR100 since the dimension of the gradient vector from the penultimate layer is $d=4096\times 100$, the size of the labeled pool is $T$=50K, and the number of examples selected in each round is $k$=5K. In order to solve this inefficiency for CIFAR100, we split the labeled pool randomly into 100 partitions and ran separate instances of the algorithm in each partition with batch size $k/100$.

Each algorithm is initialized with a seed set that is sampled uniformly at random from the pool. After that, sampling then proceeds in a series of rounds $r$ where the model is frozen until a batch $k$ of examples is selected.
The seed set size and sampling batch size $k$ are set to 1K for CIFAR10, SVHN, EUROSAT, CIFAR10 Corrupted, Fashion MNIST, and to 5K for CIFAR100.
The experiment was repeated for 5 trials. Any trial that encountered divergent training, i.e.\ the resulting model accuracy is more than three times below the standard error of model's accuracy on seed set, was dropped. We note that this happened infrequently (less than 3\% of the time) and all reported averaged results have at least 3 trials.
\looseness=-1

Figure \ref{fig:vgg16} shows the mean and standard error of VGG16 model's accuracy on a held out test set comparing \IWEBS-dis to the baseline methods. The \IWEBS-dis algorithm either outperforms or matches the performance of the baseline algorithms for all datasets. We also find that margin sampling consistently performs well against the remaining baseline algorithms and that BADGE either matches the performance of margin sampling or slightly underperforms on some datasets (Eurosat, Fashion MNIST, CIFAR100). Coreset admits a similar and at times slightly poorer performance compared to random sampling.

Next, Figure \ref{fig:vgg16-iwal} compares the two variants of our algorithm: \IWEBS-dis and \IWEBS-ent.
We find that the \IWEBS-dis performs slightly better than \IWEBS-ent on most of the datasets.
This is not surprising since the \IWEBS-dis sampling probability leverages label information and more computational power, i.e. trains two models. As explained in Section~\ref{sec:theory}, it also better fits our theoretical motivation. Nevertheless, it is important to note that \IWEBS-ent, without the label information, still consistently outperforms or matches the performance of the baselines for all the datasets.

\begin{figure}[t]
\centering
\begin{tabular}{ccc}
\includegraphics[width=0.32\textwidth]{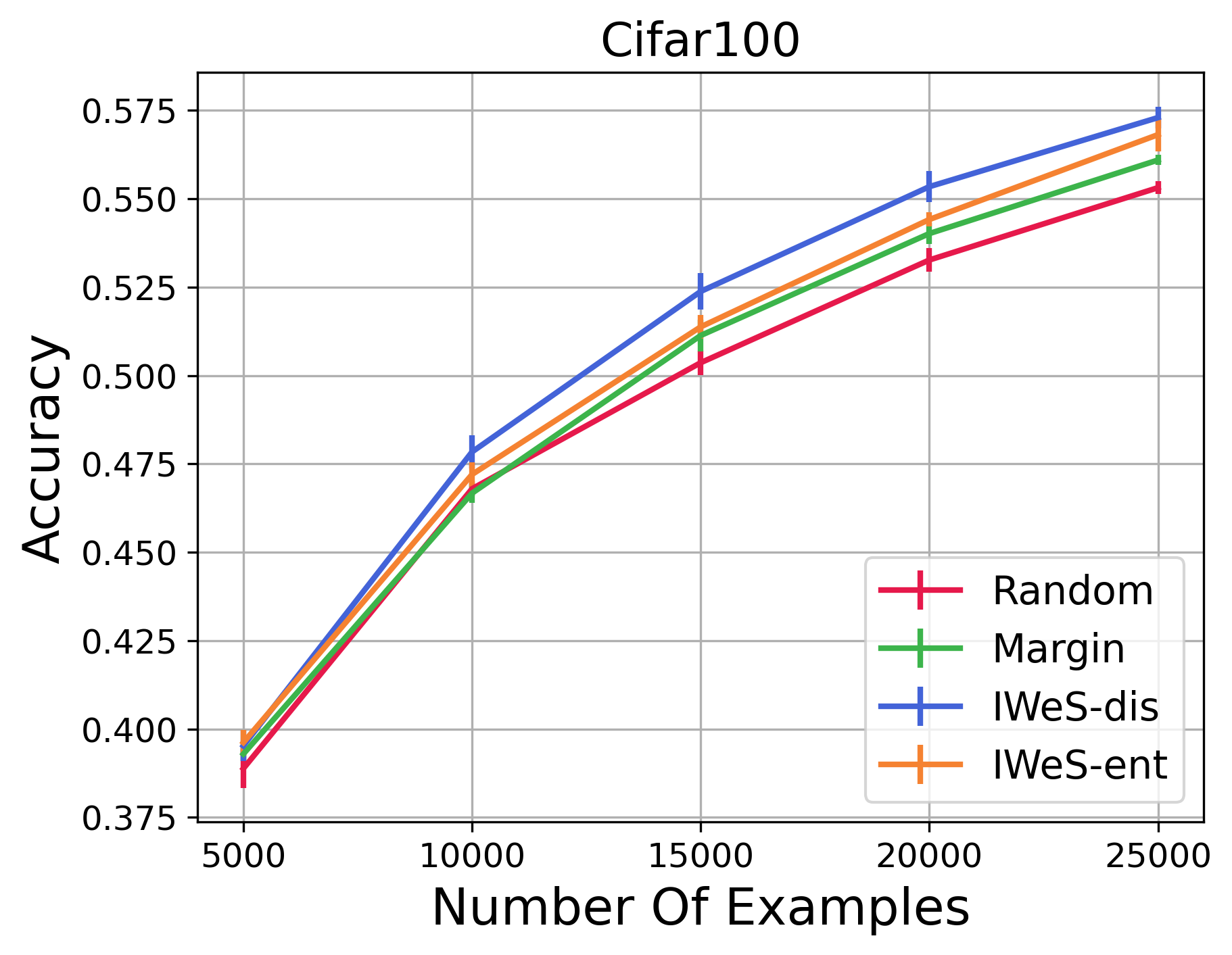}
&
\includegraphics[width=0.32\textwidth]{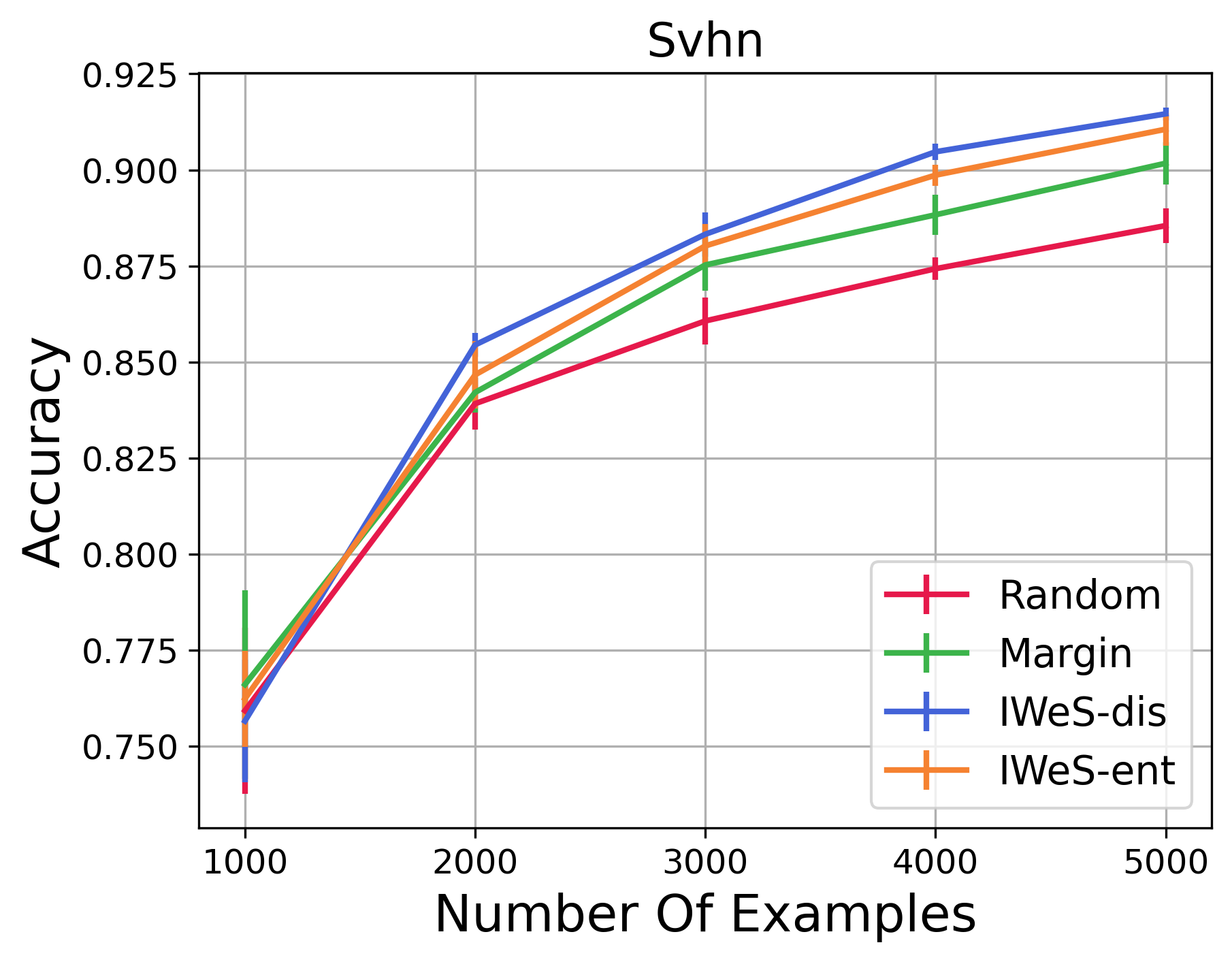}
&
\includegraphics[width=0.32\textwidth]{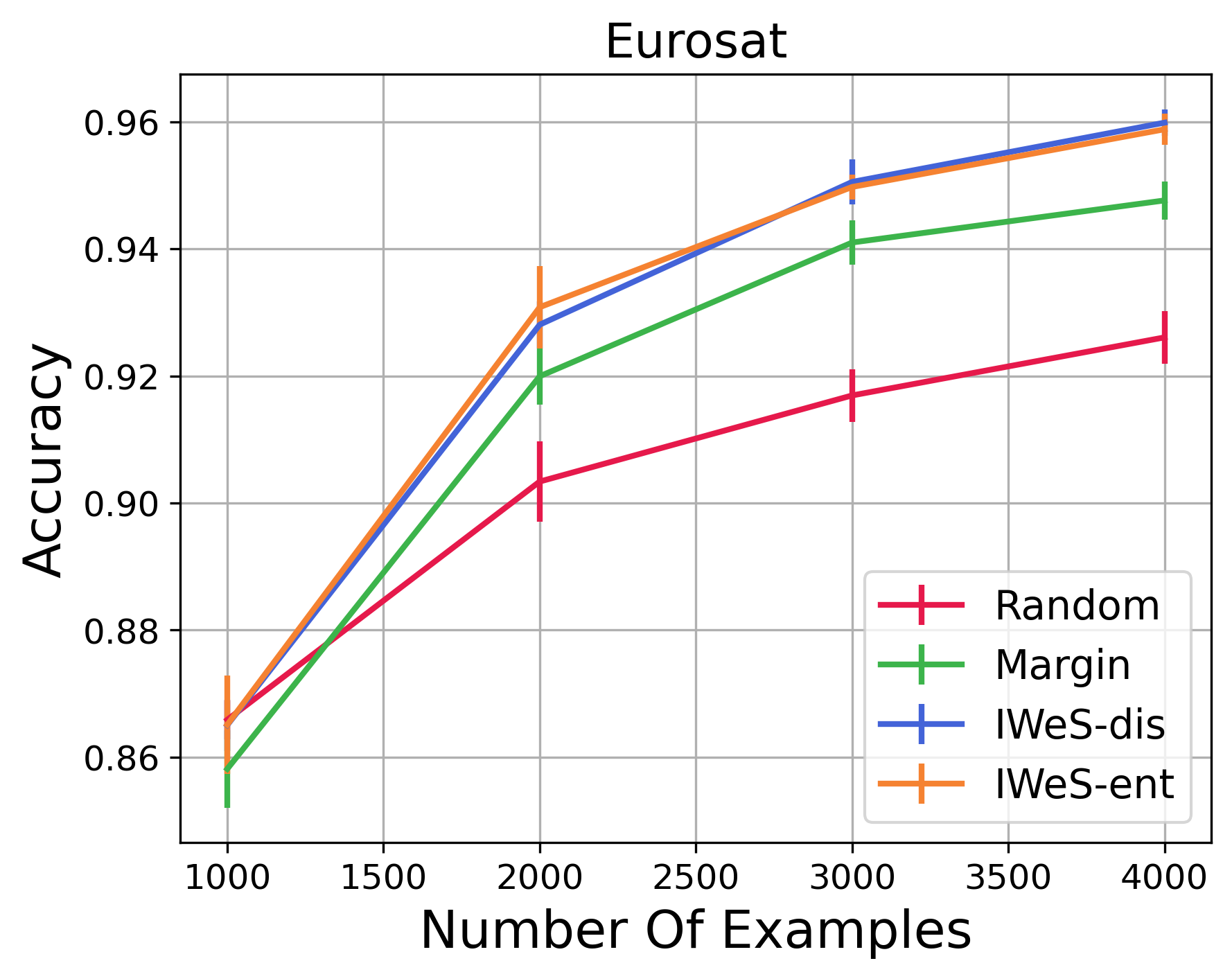}
\end{tabular}
\caption{Accuracy of VGG16 when trained on examples selected by \IWEBS-ent, \IWEBS-dis, margin sampling and random sampling.
}
\vspace{-15pt}
\label{fig:vgg16-iwal}
\end{figure}

\subsection{Multi-label Open Images Experiments}
\vspace{-0.5em}

\begin{wrapfigure}{R}{0.42\textwidth}
\centering
\vspace{-10pt}
\includegraphics[width=0.42\textwidth]{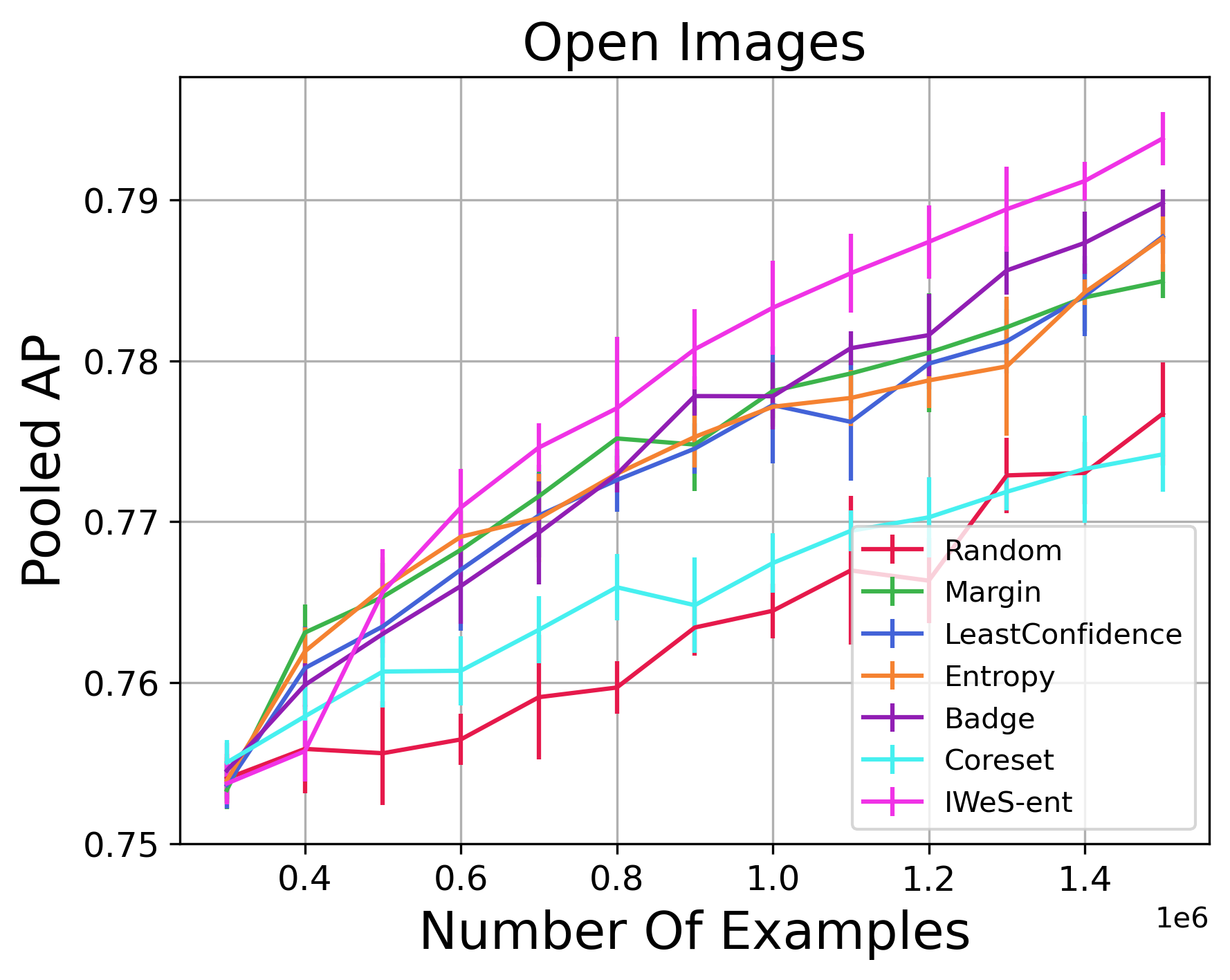}
\vspace{-20pt}
\caption{Pooled Average Precision of ResNet101 trained on examples selected by \IWEBS-ent and the baseline algorithms.}
\vspace{-10pt}
\label{fig:resnet101-openimages}
\end{wrapfigure}

In this section, we evaluate the performance of the \IWEBS\ algorithm on Open Images v6.  We train a ResNet101 model on 64 Cloud two core TPU v4 acceleration, and apply batch SGD with batchsize of 6144 and an initial learning rate of $10^{-4}$ with decay logarithmically every $5\times 10^8$ examples. We add a global pooling layer with a fully connected layer of 128 dimensions as the final layers of the networks, which is needed by BADGE and Coreset baselines.  The model is initialized with weights that were pre-trained on the validation split using 150K SGD steps, and at each sampling round, the model is trained on all past selected examples with an additional 15K SGD steps.

In the previous section, our results show that the \IWEBS-dis algorithm only slightly outperforms the \IWEBS-ent algorithm on a few datasets. Additionally, since the \IWEBS-dis requires training two neural networks, which is computationally expensive in this scenario, we only test the performance of \IWEBS-ent. Since \IWEBS-ent doesn't use the label information, this also allows us to measure the performance of the algorithm in an active learning setting. \looseness=-1

Since Open Images is a multi-label dataset, the sampling algorithms must not only select the image, but also the class. That is, each example selected by an algorithm consists of an image-class pair with a corresponding binary label indicating whether the corresponding class is present in the image or not. In order to adapt \IWEBS-ent to the multi-label setting, the entropy sampling probability for each image-class pair is defined as $p(\x,\cdot)=-\P_{f_r}(y|\x)\log_2{\P_{f_r}(y|\x)}-\br{1-\P_{f_r}(y|\x)}\log_2\br{1-\P_{f_r}(y|\x)}$, where $\P_{f_r}(y|\x)$ is the model class probability of a positive label at round $r$. A seed set of size 300K is sampled uniformly at random from the pool, and at each sampling round $r$, the algorithms select 100K examples.  Similarly to the previous section, in order to run BADGE on Open Images, we divide the pool into 100 partitions and run separate instances of the algorithm in each partition.  For \IWEBS , the weight capping parameter is set to 10.

Figure~\ref{fig:resnet101-openimages} shows the mean and standard error across 5 trials of the pooled average precision (Pooled AP) metric for each algorithm. As the number of selected examples increases, \IWEBS-ent outperforms all other baselines methods on the Open Images dataset. We also find that BADGE performs similarly or even slightly worse than the uncertainty-based sampling algorithms when the number of selected examples is smaller than 800K, and then outperforms all uncertainty-based sampling as the number of selected examples increases. Coreset initially performs better than random sampling, but at later sampling rounds, it admits a similar performance to random sampling.

\section{Theoretical Motivation}\label{sec:theory}
 
In order to theoretically motivate the \IWEBS\ algorithm, we analyze a closely related algorithm which we call \IWEBSV , adapted from the IWAL algorithm of \citet{beygelzimer2009importance}. We prove that \IWEBSV\ admits generalization bounds that scale with the dataset size $T$ and sampling rate bounds that are in terms of a new disagreement coefficient tailored to the subset selection framework.

Below, we let $L(h)=\E_{(\x,y)\sim\cD}[\ell(h(\x),y)]$ denote the expected loss of hypothesis $h\in\cH$, and $h^*=\argmin_{h\in\cH}L(h)$ be the best-in-class hypothesis.  Without loss of generality, we consider a bounded loss $\ell:\cZ\times\cY\rightarrow [0,1]$ mapping to the interval $[0,1]$. Such a loss can be achieved by any bounded loss after normalization.  For simplicity, we assume $\cH$ is a finite set, but our results can be easily extended by standard covering arguments to more general hypothesis sets such as finite VC-classes.
\looseness=-1

The \IWEBSV\ algorithm operates on an i.i.d. example $(\x_1,y_1),(\x_2,y_2),\ldots,(\x_T,y_T)$ drawn from $\cD$ sequentially. It maintains a version space $\cH_t$ at any time $t$, with $\cH_1=\cH$. At time $t$, \IWEBSV\ flips a coin $Q_t\in\bc{0,1}$ with bias $p_t$ defined as
\begin{align}\label{eq:pt_iwebsv}
p_t=\max_{f,g\in \cH_t}\ell(f(\x_t),y_t)-\ell(g(\x_t),y_t)
\end{align}
where $\cH_t=\bc{h\in\cH_{t-1}:\frac{1}{t}\sum_{s=1}^t\frac{Q_s}{p_s}\ell(h(\x_s),y_s)\leq \min_{h'\in\cH_{t-1}}\frac{1}{t}\sum_{s=1}^t\frac{Q_s}{p_s}\ell(h'(\x_s),y_s)+\Delta_{t-1}}$ with $\Delta_{t-1}=\sqrt{\frac{8\log(2T(T+1)|\cH|^2/\delta)}{t-1}}$. The example is selected if $Q_t=1$ and otherwise it is discarded. The main idea behind this algorithm is thus to define a sampling probability that is in terms of the disagreement between two hypothesis $f,g$ that are not too far from the best model trained on the past selected data, i.e. $\min_{h\in\cH_{t-1}}\frac{1}{t}\sum_{s=1}^t\frac{Q_s}{p_s}\ell(h(\x_s),y_s)$. The formal \IWEBSV\ algorithm pseudo-code (Algorithm \ref{algo:pre-iwals}) and all the theorem proofs can be found in Appendix \ref{sec:proof}.

For general, e.g.\ non-linear, hypothesis classes it is
computationally infeasible to find two hypotheses $f,g \in \cH_t$ that maximize the expression in equation~\eqref{eq:pt_iwebsv}.
This main impracticality of \IWEBSV\ is reason why we developed the \IWEBS\ algorithm of the previous section.
This drawback is also shared by the \IWAL\ algorithm of \cite{beygelzimer2009importance}, which computes a sampling probability very similar to that of equation \eqref{eq:pt_iwebsv}, but with an additional maximization over the choice of $y \in \cY$ in the definition of the sampling probability $p_t$.

Before continuing we explain how our practical algorithm \IWEBS-dis, specifically using sampling probability in equation \eqref{eq:disagreement_strategy}, is closely related to the \IWEBSV\ algorithm. Recall that the \IWEBS\ algorithm trains two models $f$ and $g$ each minimizing the importance-weighted loss using the data sampled so far. Therefore, each model exhibits reasonable training loss, i.e. they are expected to be included in the version space $\cH_t$ of good hypothesis, while the different random initializations (in the case of non-convex neural network hypotheses) results in models that still differ in certain regions of the feature space. Thus, the difference in equation \eqref{eq:disagreement_strategy} can be thought of as a less aggressive version of the difference found in the maximization of equation \eqref{eq:pt_iwebsv}.

Another dissimilarity between the two is that the \IWEBS-dis algorithm is defined for the batch streaming setting while the \IWEBS-dis algorithm and its analysis is developed for the streaming setting. Said differently, the \IWEBSV\ algorithm can be seen as a special case of the \IWEBS-dis algorithm with target subset size of 1. To extend the theoretical guarantees of \IWEBSV\ to the batch streaming setting, we can follow a similar analysis developed by \cite{amindesalvorostami2020} to also find that the effects of delayed feedback in the batch streaming setting are in fact mild as compared to the streaming setting.
\subsection{Generalization bound}\label{sec:genbound}
\vspace{-0.5em}
Next, we turn to the topic of generalization guarantees and we review an existing bound for coreset based algorithms.
The guarantees of coreset algorithms are generally focused on showing that a model's training loss on the selected subset is close to the same model's training loss on the whole dataset. That is, given dataset $\cP=\{(\x_i,y_i)\}_{i=1}^T\sim\cD^T$,
the learner seek to select a subset $m < T$ of examples $\cS=\{(\x'_i,y'_i)\}_{i=1}^m$ along with a corresponding set of weights $w_1,\ldots,w_m$ such that for some small $\epsilon>0$ and for all $h\in\cH$, the \emph{additive error coreset guarantee} holds
$\left| \sum_{i=1}^m w_i \ell(h(\x_i'),y_i') -  \sum_{i=1}^T  \ell(h(\x_i),y_i)  \right| \leq \epsilon  T $.
The following proposition, which is a minor extension of Fact~8 of \cite{karninedo2019}, allows us to convert a coreset guarantee into a generalization guarantee.

\begin{restatable}{prop}{coresetapprox}
\label{th:bad_genbound}
Let  $h'=\argmin_{h\in \cH} \sum_{i=1}^m w_i \ell(h(\x_i'),y_i') $,
and let the additive error coreset guarantee hold for any $\epsilon>0$,
then for any $\delta>0$, with probability at least $1-\delta$, it holds that
$
L(h') \leq L(h^*) + 2\epsilon + 2\sqrt{\ln(4/\delta)/2T}.
$
\end{restatable}

As shown above, the generalization guarantee depends linearly on $\epsilon$ which in turn depends on the size of the subset $m$.
To give a few examples, \cite{karninedo2019} show that for hypotheses that are defined as analytic functions of dot products (e.g. generalized linear models) this dependence on $m$ is $\epsilon = O(1/m)$, while for more complex Kernel Density Estimator type models the dependence is $\epsilon = O(1/\sqrt{m})$.  See \citet{mai2021coresets}, Table 1, for examples on the dependency between $\epsilon$ and $m$ under different data distributions assumptions (e.g. uniform, deterministic, $\ell_1$ Lewis) and for specific loss functions (e.g. log loss, hinge loss).

We now provide a generalization guarantee for the \IWEBSV\ algorithm, which depends on the size of the labeled pool size $T$. The proof follows from that in \citet{beygelzimer2009importance}.

\begin{restatable}{theorem}{iwebsgenbound}\label{th:genbound}
Let $h^*\in \cH$ be the minimizer of the expected loss function $h^*=\argmin_{h\in \cH}L(h)$. For any $\delta>0$, with probability at least $1-\delta$, for any $t\geq 1$ with $t\in\bc{1,2\ldots,T}$, we have that  $h^*\in \cH_t$ and that $ L(f)-L(g) \leq 2 \Delta_{t-1}$ for any $f,g\in \cH_t$.
In particular, if $h_T$ is the output of \IWEBSV , then $L(h_T)-L(h^*) \leq 2 \Delta_{T-1} = \cO\big(\sqrt{\log(T/\delta) / T}\big).$
\end{restatable}
 Unlike the distribution-specific and loss-specific theoretical guarantees proposed in the coreset literature, Theorem \ref{th:genbound} holds for any bounded loss function and general hypothesis classes.  If we ignore log terms and consider the more complex Kernel Density Estimator class of hypotheses, the coreset method of \cite{karninedo2019} requires $m = \cO(T)$ coreset samples in order to achieve an overall $\cO(1/\sqrt{T})$ generalization bound. As we will see in the next section, the required \IWEBS\ sampling rate can also be as high as $\cO(T)$, but critically is scaled by the best-in-class loss, which in favorable cases is significantly smaller than one.

\subsection{Sampling Rate bounds}\label{sec:samplerate}
\vspace{-0.5em}
\cite{hanneke2007bound} proves that the expected number of labeled examples needed to train a model in an active learning setting can be characterized in terms of the disagreement coefficient of the learning problem.  Later, \cite{beygelzimer2009importance} generalizes this notion to arbitrary loss functions, and in this work, we further generalize this for the subset selection setting.

Recall that the disagreement coefficient $\theta_\AL$ in \cite{beygelzimer2009importance} for the active learning setting is defined as
\vspace{-10pt}
\begin{align*}
\theta_\AL=\sup_{r\geq 0}\frac{\E_{\x\sim\cX}\left[\max_{h\in \cB_\AL(h^*,r)}\max_{y\in\cY}| \ell(h(x),y) - \ell(h^*(x),y)|\right]}{r},
\end{align*}
where $\cB_\AL(h^*,r)=\{h\in\cH:\rho_\AL(h,h^*)\leq r\}$ with $\rho_\AL(f,g)=\E_{\x\sim\cX}[\sup_{y\in\cY}|\ell(f(\x),y)-\ell(g(\x),y)|]$.
Informally, this coefficient quantifies how much disagreement there is among a set of classifiers that is close to the best-in-class hypothesis. In the subset selection setting, labels are available at sample time and, thus, we are able to define the following disagreement coefficient:
\begin{definition} Let $\rho_\IWALSone(f,g) =  \E_{(\x,y)\in \cD} [ | \ell(f(\x),y) - \ell(g(\x),y) | ]$ and  $\cB_\IWALSone(h^*,r)=\{ h\in \cH: \rho_\IWALSone(h,h^*)\leq r \}$ for $r\geq0$.
The disagreement coefficient in the subset selection setting is defined as
\vspace{-5pt}
\begin{align*}
\theta_\IWALSone=\sup_{r\geq 0}\frac{\E_{(\x,y)\sim\cD}\left[\max_{h\in \cB_\IWALSone(h^*,r)}| \ell(h(\x),y) - \ell(h^*(\x),y)|\right]}{r}.
\end{align*}
\end{definition}
\vspace{-10pt}
The main difference between the above coefficient and that of  \cite{beygelzimer2009importance} is that there is no supremum  over all label $y\in \cY$ both in the definition of the distance $\rho$ and the coefficient's numerator. Instead, the supremum is replaced with an expectation over the label space.

The following theorem leverages $\theta_\IWALSone$ to derive an upper bound on the expected number of selected examples for the \IWEBSV\ algorithm. Below, let $\cF_t=\{(\x_i,y_i, Q_i)\}_{i=1}^t$ be the observations of the algorithm up to time $t$.
\begin{restatable}{theorem}{coresetiwalsone}\label{th:coreset_iwals1}
For any $\delta>0$, with probability at least $1-\delta$, the expected sampling rate of the \IWEBSV \ algorithm is:
$
\sum_{t=1}^T\E_{(\x_t,y_t)\sim\cD} \bs{p_t\big |\cF_{t-1}}
=\cO\Big(\theta_\IWALSone \br{ L(h^*) T +  \sqrt{T \log({T/\delta})}}\Big).
$
\end{restatable}
Suppressing lower order terms, the above expected sampling rate bound is small whenever the product of the disagreement coefficient and the expected loss of the best-in-class is small.  In such cases, by combining the above theorem with the generalization guarantee, it holds that \IWEBSV\ returns a hypothesis trained on a only fraction of the points that generalizes as well as a hypothesis trained on the full dataset of size $T$.  Theorem \ref{th:coreset_iwals1} can be further improved by adapting the ideas found in \citet{cortesdesalvogentilmohrizhang2019} to the \IWEBSV \ algorithm. See Appendix~\ref{sec:enhanced} for this enhanced analysis.

The form of this sampling rate bound is similar to that of \cite{beygelzimer2009importance}. More concretely, under the assumption that a loss function has bounded slope asymmetry, that is  $K_{\ell} = \sup_{z,z'\in\cZ}\frac{\max_{y\in\cY}|\ell(z,y)-\ell(z',y)|}{\min_{y\in\cY}|\ell(z,y)-\ell(z',y)|}$  is bounded, with probability at least $1-\delta$, the expected number of examples selected by the IWAL algorithm is given by $\cO\br{\theta_\AL K_{\ell} \br{ L(h^*)T+  \sqrt{T \log({T/\delta})}}}$. Thus, the main difference between the sampling rate bound of the IWAL algorithm and  the \IWEBSV\ algorithm are the factors that depends on the two disagreement coefficients:  $\theta_\AL K_{\ell}$ and $\theta_\IWALSone$. Since $\theta_\IWALSone$ leverages the label information we may expect it to provide a tighter bound, compared to using the label-independent disagreement $\theta_\AL$. Theorem~\ref{th:coresetiwals1_comp} shows that this is indeed the case.

\begin{restatable}{theorem}{coresetiwalsonecomp}\label{th:coresetiwals1_comp}
If the loss function has a bounded slope asymmetry $K_{\ell}$,
then $\theta_\IWALSone \leq\theta_\AL K_{\ell}$.
\end{restatable}
The above theorem in conjunction with the sampling rate guarantees thus proves that the sampling rate bound of \IWEBS\ of Theorem~\ref{th:coreset_iwals1} is tighter than the sampling rate bound of the IWAL algorithm.

\section{Conclusion}
In this paper we have introduced a subset selection algorithm, \IWEBS\, that is designed for arbitrary hypothesis classes including deep networks. We have shown that the \IWEBS\ algorithm outperforms several natural and important baselines across multiple datasets. In addition, we have developed an initial theoretical motivation for our approach based on the importance weighted sampling mechanism.  A natural next step is enforcing a notion of diversity as it will likely provide improved performance in the large-batch sampling setting and thus, we plan to adapt the diversity-based method in \cite{citovsky2021} by replacing the uncertainty sampling component with the \IWEBS\ algorithm.

\bibliography{ref}
\bibliographystyle{iclr2023_conference}

\newpage
\appendix
\begin{center}
\Large {Supplementary Material}
\end{center}

\section{Extended Algorithmic and Experimental Details}\label{sec:extexp}

\subsection{Dataset Details}
Table~\ref{tab:datasets} and Table~\ref{tab:openimages} provide further details of the six multi-class classification datasets and the multi-label Open Images dataset, respectively.
\begin{table}[!htbp]
\centering
\begin{tabular}{c|c|c|c|c|c}
~ & Train & Test & \# Classes & Image Size & Description \\
\hline
CIFAR10 & 50,000 & 10,000 & 10 & 32$\times$32$\times$3 & classify the object in the image\\
SVHN & 73,257 & 26,032 & 10 & 32$\times$32$\times$3 & classify street view house number \\
CIFAR Corrupted & 7,614 & 2,386 & 10  & 32$\times$32$\times$3 & classify the corrupted object in the image\\
Eurosat & 8,000 & 5,000 & 10 & 64$\times$64$\times$3 & classify land use and land cover satelite image\\
Fashion MNIST & 60,000 & 10,000 & 10 & 32$\times$32$\times$3 & classify the type of clothes\\
CIFAR100 & 50,000 & 10,000 & 100 & 32$\times$32$\times$3 & classify the object in the image\\
\end{tabular}
\caption{Multi-class Classification Datasets statistics}
\label{tab:datasets}
\end{table}
\begin{table}[!htbp]
\centering
\begin{tabular}{c|c|c|c}
~ & Images & Positives & Negatives \\
\hline
Train & 9,011,219 & 19,856,086 & 37,668,266 \\
Validation & 41,620 & 367,263 & 228,076 \\
Test & 125,436 & 1,110,124 & 689,759 \\
\end{tabular}
\caption{Open Images Dataset v6 statistics by data split}
\label{tab:openimages}
\end{table}

\subsection{More Experiment Details and Results}

When running \IWEBS, at each iteration, we pass over the labeled pool sequentially in a uniform random order, removing each selected example from the pool. If we exhaust the entire pool before selecting $k$ examples, we start again at the beginning of the sequence and iterate over the remaining examples. In order to reduce the number of passes required for the smaller datasets, we scale the sampling probabilities of the remaining points uniformly by $1 + \frac{j}{10}$, where $j$ is the number of passes so far.

Figure~\ref{fig:vgg16-iwal-additional} compares \IWEBS-dis and \IWEBS-ent for additional three datasets that we omit in Section~\ref{sec:exp}.

\begin{figure}[!htbp]
\begin{tabular}{ccc}
\includegraphics[width=0.32\textwidth]{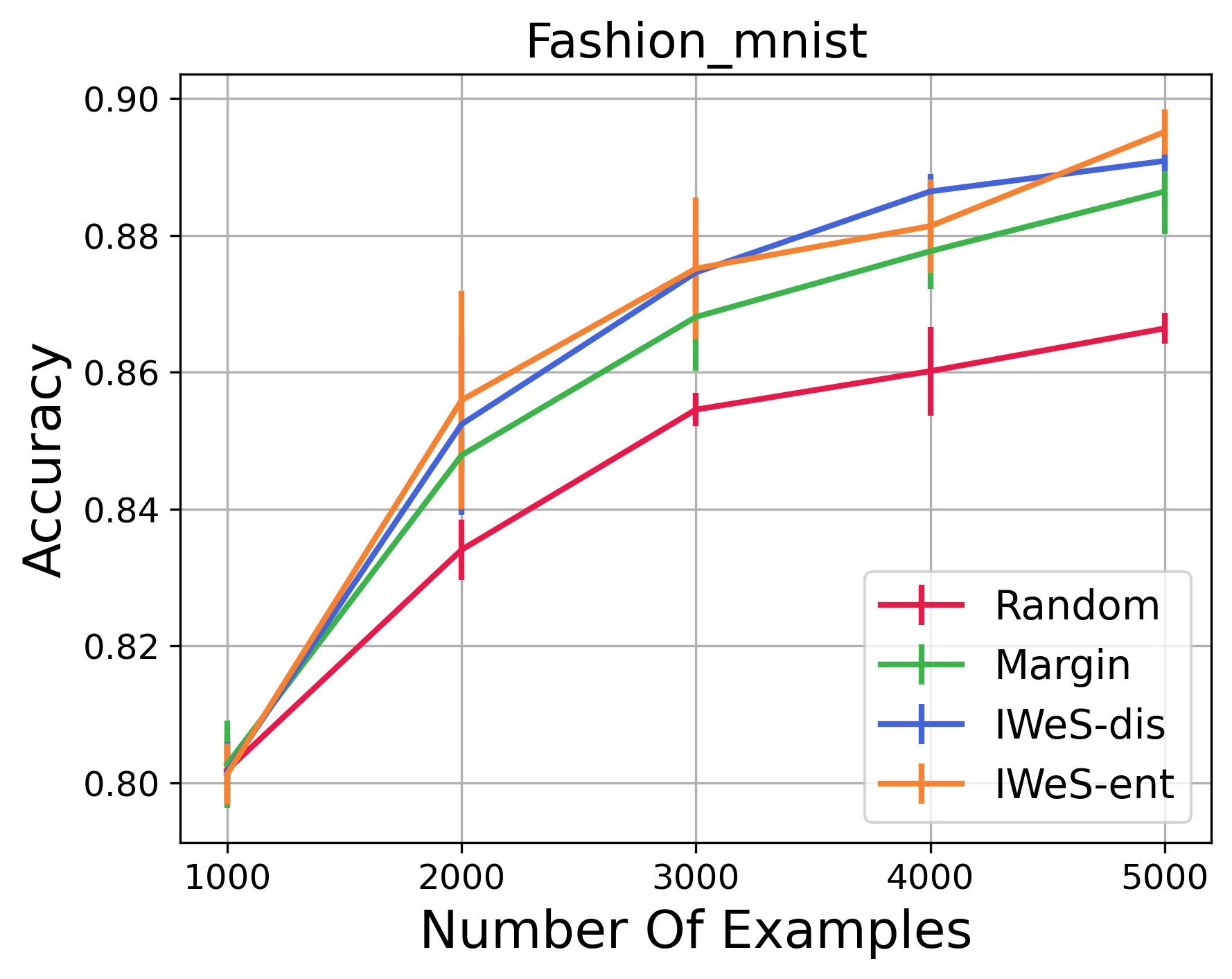}
&
\includegraphics[width=0.32\textwidth]{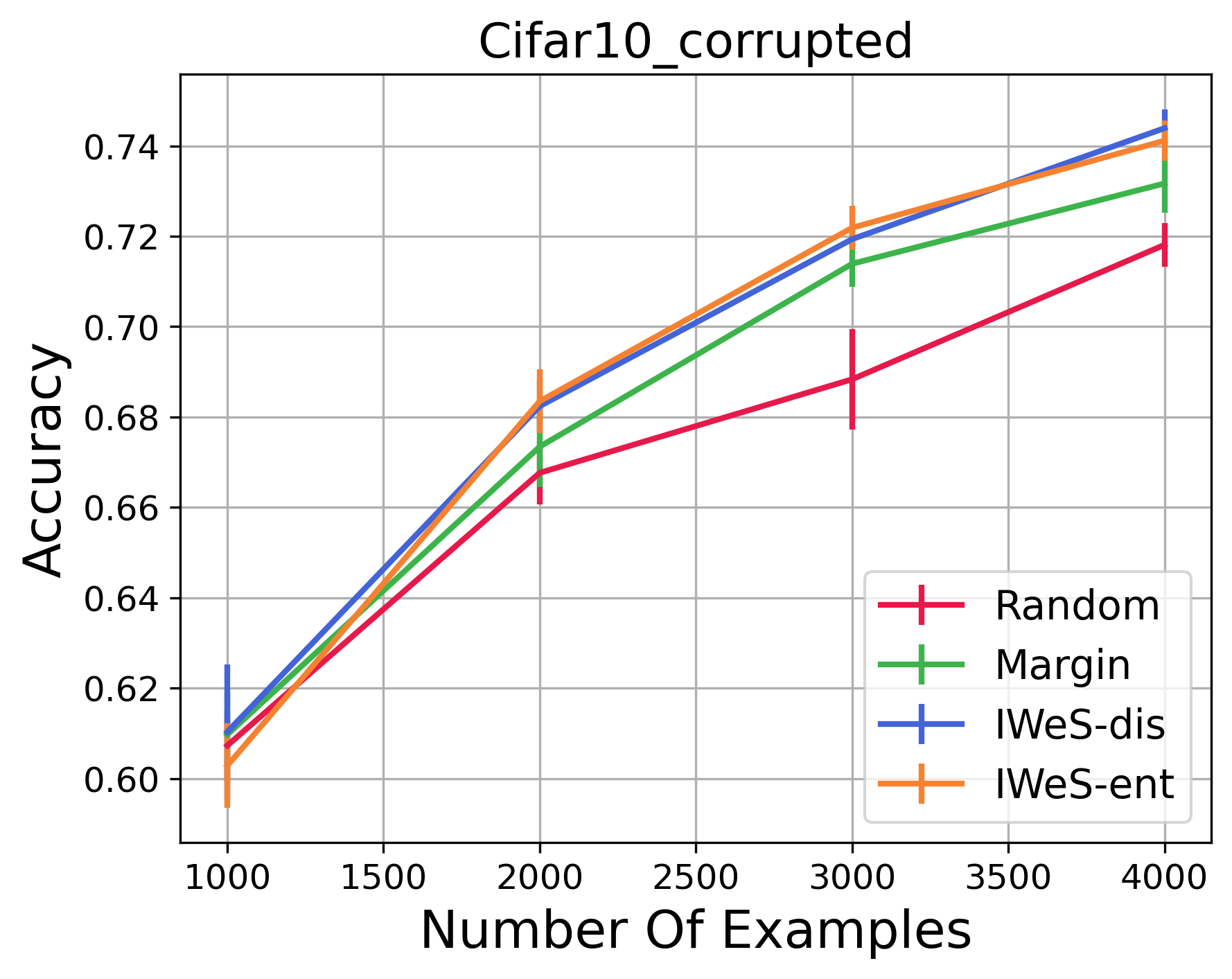}
&
\includegraphics[width=0.32\textwidth]{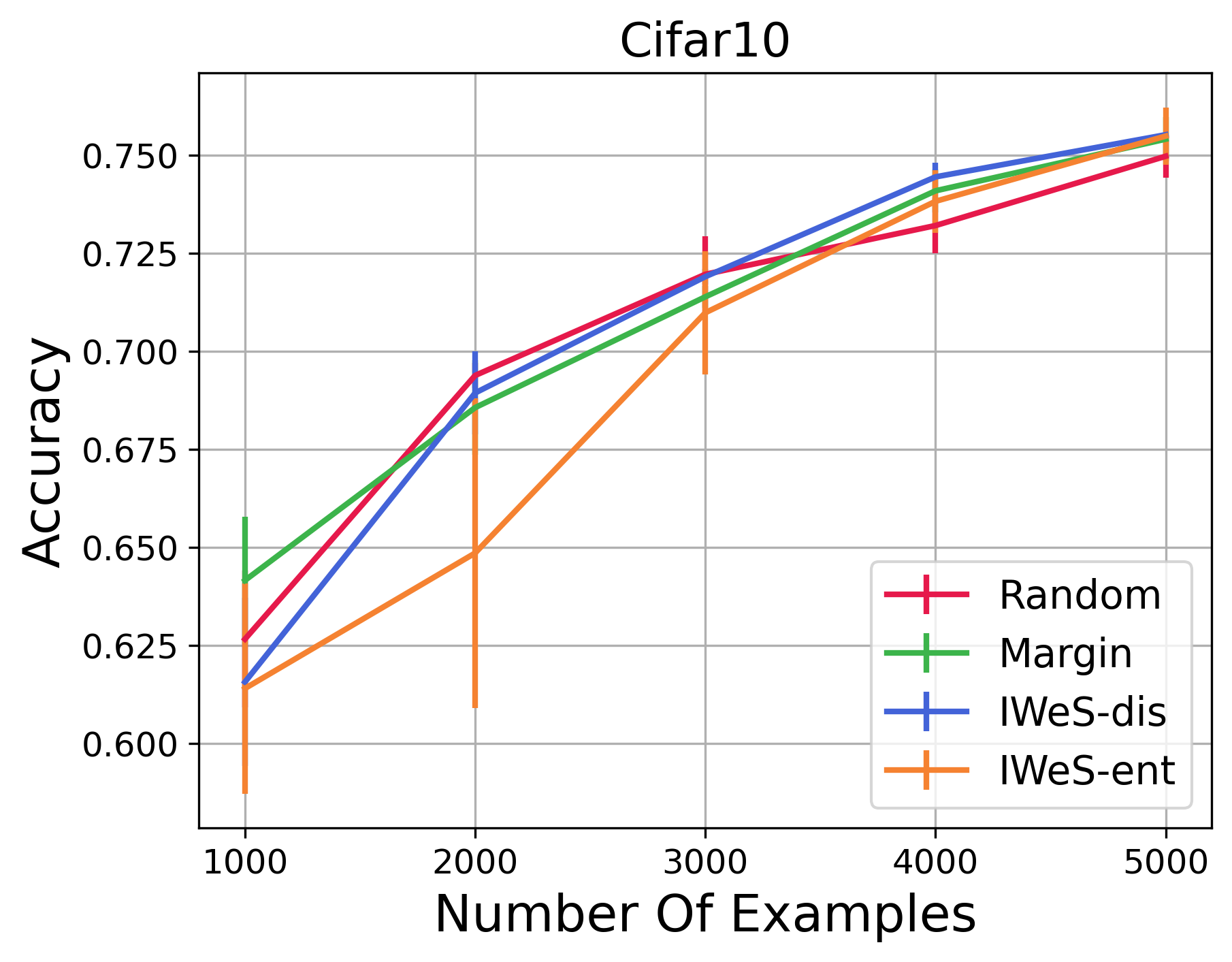}
\end{tabular}
\caption{Accuracy of VGG16 trained on examples selected by \IWEBS-ent, \IWEBS-dis, margin sampling and random sampling.
}
\label{fig:vgg16-iwal-additional}
\end{figure}

\subsection{Efficient version of entropy-based disagreement}
Here we develop an efficient version of entropy-based disagreement which we call  \IWEBS-loss by defining the sampling probability to be proportional to the model's entropy restricted to the labeled examples $(\x,y)$ as follows:
\begin{align}\label{eq:loss_strategy}
p(\x,y)=-\P_{f_{r}}(y|\x)\log{\P_{f_{r}}(y|\x)}.
\end{align}
where $\P_{f_{r}}(y|\x)$ is the probability of class $y$ with model $f_{r}$ given example $\x$. As $\P_{f_{r}}(y|\x)$ increase from 0 to 1, this sampling probability first increase then decrease. Thus the sampling probability is high whenever the model is not confident about the model prediction. Unlike \IWEBS-dis, this definition only requires training one model, thereby saving some computational cost.
 
Figure~\ref{fig:vgg16-iwesloss} compares the two variants \IWEBS-dis and \IWEBS-loss. We find that the performance are similar across all datasets, with \IWEBS-loss behave slightly better than \IWEBS-dis at the first two sample iterations (i.e. Fashion MNIST, CIFAR10, Eurosat), and \IWEBS-dis slightly outperform \IWEBS-loss as the number of sample iteration increases (i.e. SVHN, Eurosat).

\begin{figure}[!htbp]
\begin{tabular}{ccc}
\includegraphics[width=0.32\textwidth]{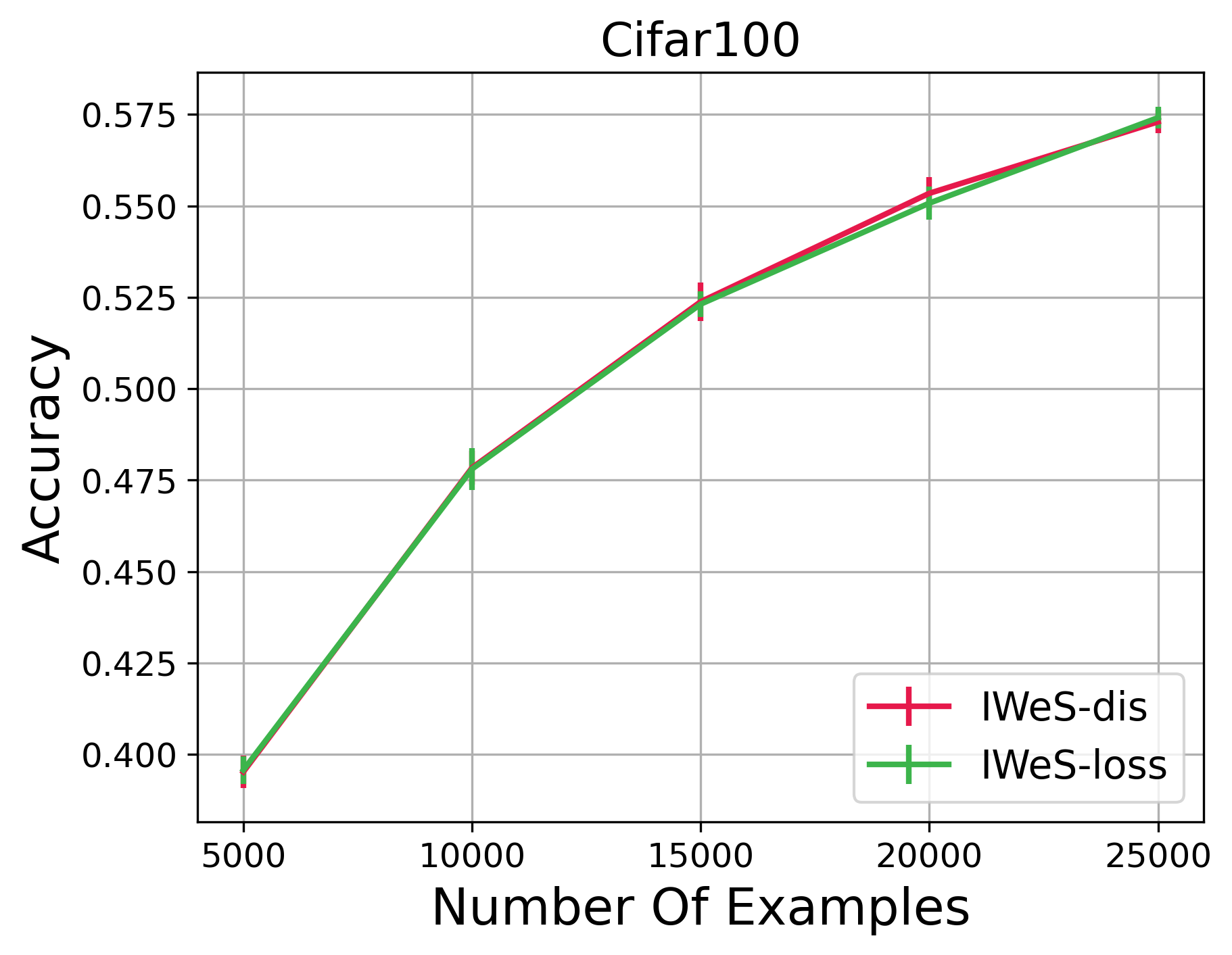}
&
\includegraphics[width=0.32\textwidth]{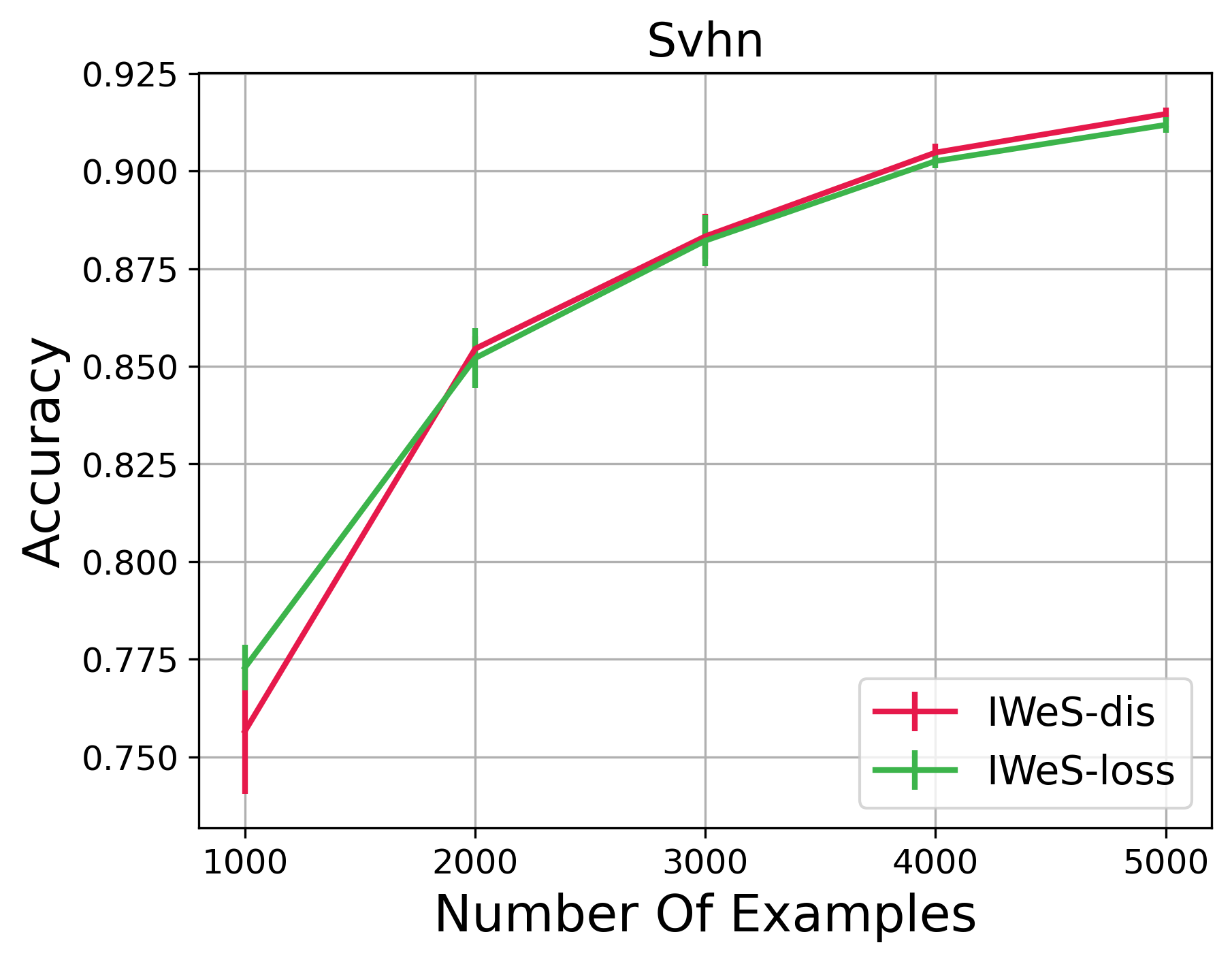}
&
\includegraphics[width=0.32\textwidth]{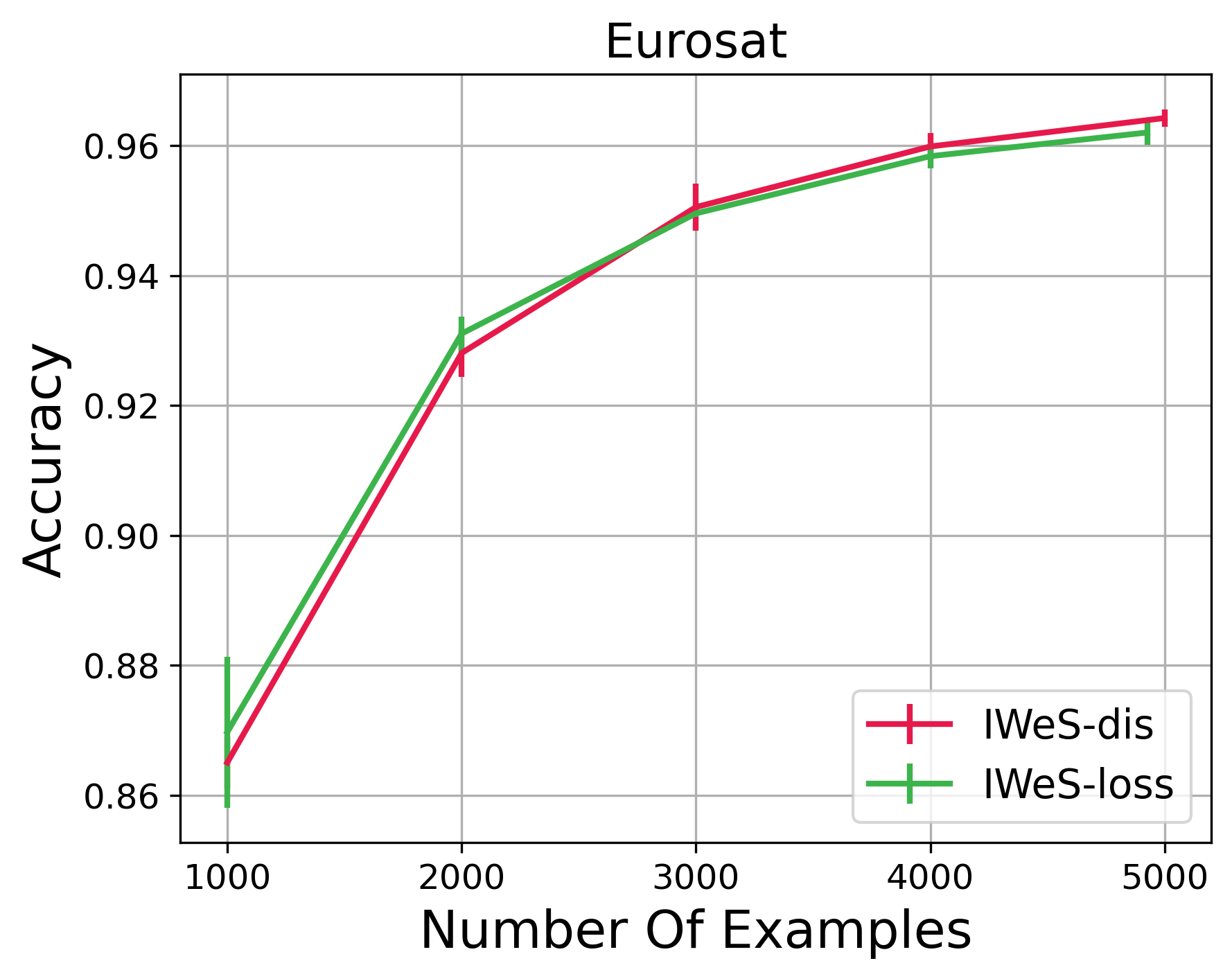}\\
\includegraphics[width=0.32\textwidth]{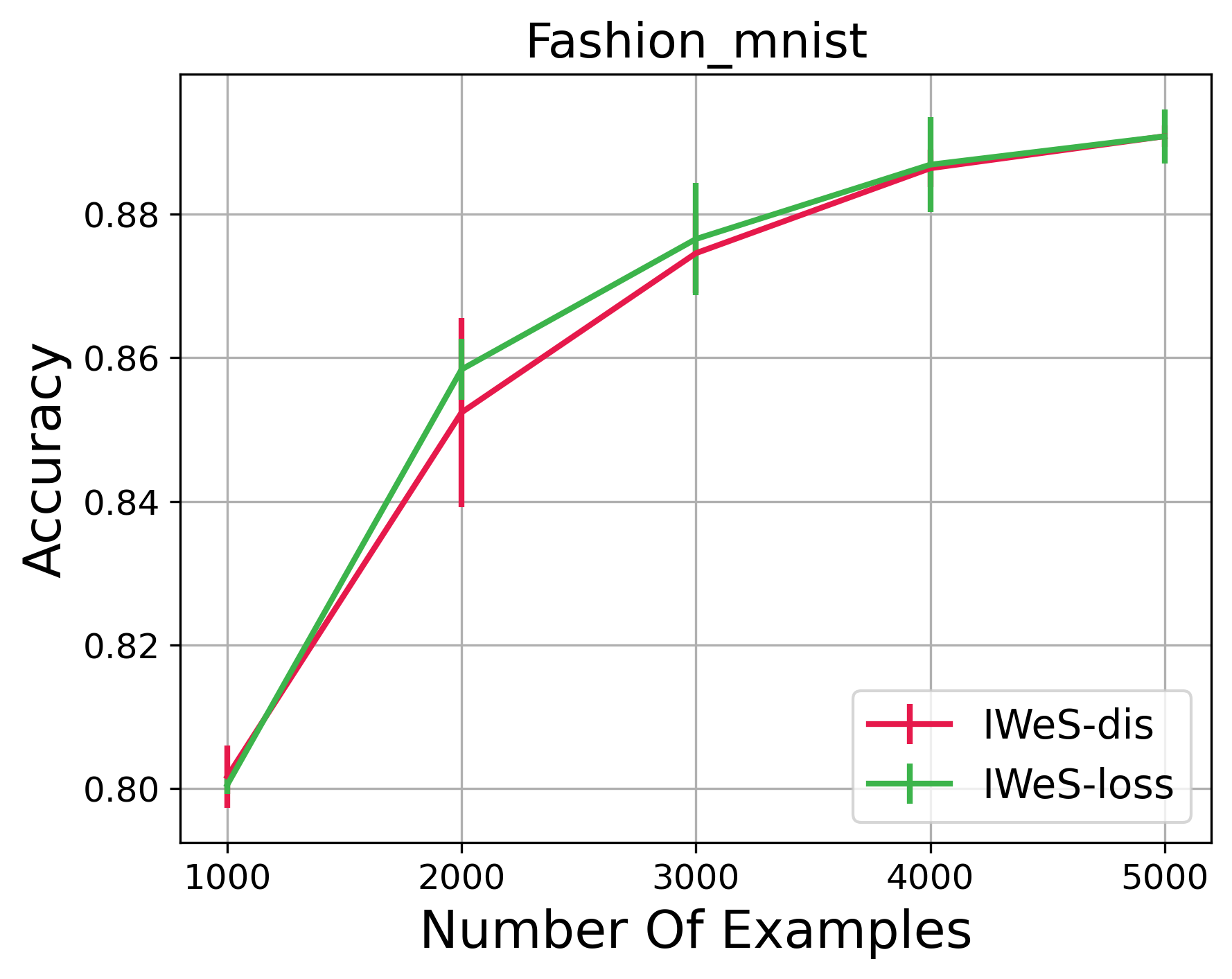}
&
\includegraphics[width=0.32\textwidth]{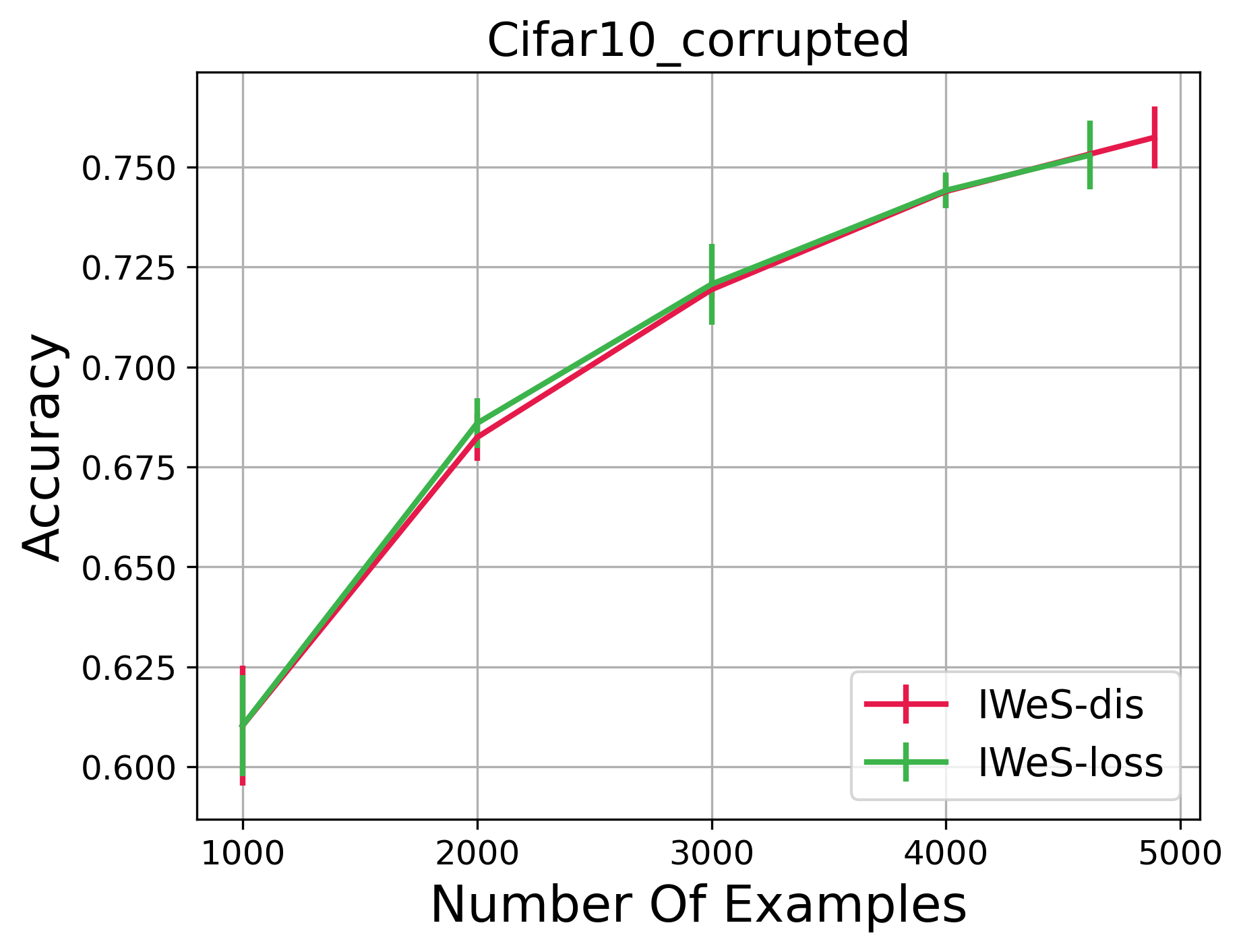}
&
\includegraphics[width=0.32\textwidth]{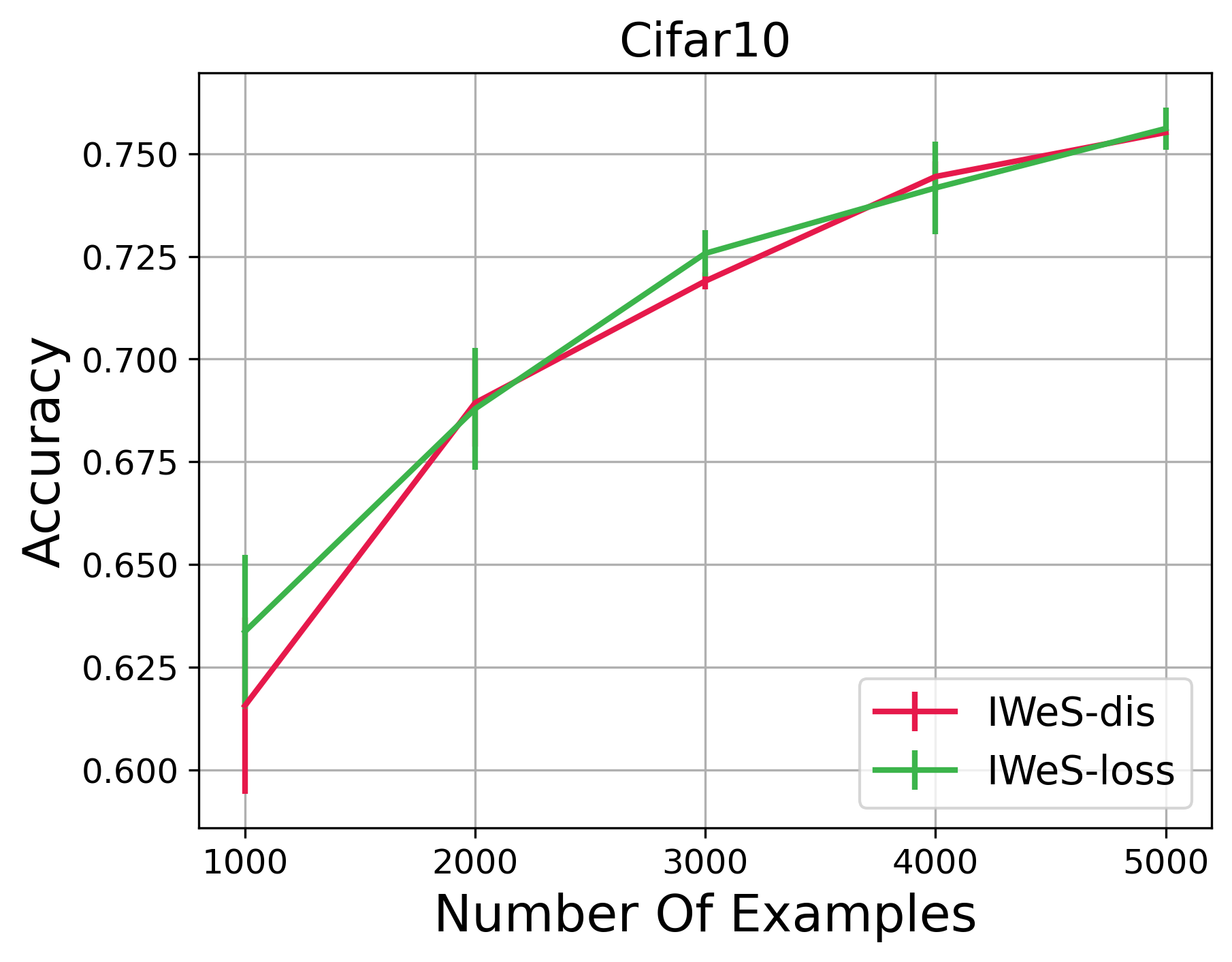}
\end{tabular}
\caption{Accuracy of VGG16 trained on examples selected by \IWEBS-loss and \IWEBS-dis
}
\label{fig:vgg16-iwesloss}
\end{figure}

\section{Proofs of Theoretical Guarantee}
\label{sec:proof}

\subsection{Proof of Proposition \ref{th:bad_genbound}}

\coresetapprox*
\begin{proof}[Proof of Proposition \ref{th:bad_genbound}]
\begingroup
\allowdisplaybreaks
\begin{align*}
L(h')&=\E_{(\x,y)\sim\cD}[\ell(h'(\x),y)]\\
&\leq\frac1T\sum_{i=1}^T\ell(h'(\x_i),y_i)+\sqrt{\frac{\ln(4/\delta)}{2T}}\tag{Hoeffding inequality}\\
&\leq\frac1T\sum_{i=1}^m w_i\ell(h'(\x_i'),y_i')+\epsilon+\sqrt{\frac{\ln(4/\delta)}{2T}} \tag{$0\leq\ell(\cdot)\leq 1$}\\
&\leq \frac1T\sum_{i=1}^m w_i\ell(h^*(\x_i'),y_i)+\epsilon+\sqrt{\frac{\ln(4/\delta)}{2T}} \tag{By the definition of $h'$}\\
&\leq \frac1T\sum_{i=1}^T\ell(h^*(\x_i),y_i)+2\epsilon+\sqrt{\frac{\ln(4/\delta)}{2T}}\\
&\leq \E_{(\x,y)\sim\cD}[\ell(h^*(\x),y)]+2\epsilon+2\sqrt{\frac{\ln(4/\delta)}{2T}} \tag{Hoeffding inequality}\\
&= L(h^*)+2\epsilon+2\sqrt{\frac{\ln(4/\delta)}{2T}}
\end{align*}
\endgroup

\end{proof}

\subsection{Proofs for statements in Section \ref{sec:genbound}}

Algorithm \ref{algo:pre-iwals} contains a detailed pseudocode of the \IWEBSV\ algorithm referred to in the main body of the paper.
\begin{algorithm}[!htbp]
\caption{\IWEBSV}
\centering
\begin{algorithmic}[1]\label{algo:pre-iwals}
\REQUIRE Labeled Pool $\cP$.
\STATE Initialize: $\cS_0=\emptyset, \cH_0=\cH$.
\FOR{$t=1,2,\ldots,T$}
\STATE Sample $(\x_t,y_t)$ uniformly at random from $\cP$.
\STATE Update:
\vspace{-15pt}
\begin{align*}
\cH_t&=\left\{h\in \cH_{t-1}: \frac{1}{t-1} \sum_{i=1}^{t-1} \frac{Q_i}{p_i} \ell(h(\x_i),y_i)\leq \min_{h\in \cH_{t-1}}\frac{1}{t-1}\sum_{i=1}^{t-1}\frac{Q_i}{p_i}\ell(h(\x_i),y_i)+\Delta_{t-1}\right\}
\end{align*}
where $\Delta_{t-1}=\sqrt{\frac{8\log(2T(T+1)|\cH|^2/\delta)}{t-1}}$.
\STATE Set $p_t=\max_{f,g\in \cH_t}\ell(f(\x_t),y_t)-\ell(g(\x_t),y_t)$.
\STATE $Q_t\sim\textrm{Bernoulli}(p_t)$.
\IF{$Q_t=1$}
\STATE Set $\cS_t=\cS_{t-1}\cup\bc{\br{\x_t,y_t,\frac{1}{p_t}}}$
\ELSE
\STATE $\cS_t=\cS_{t-1}$.
\ENDIF
\STATE $h_t=\arg\min_{h\in H}\sum_{(\x,y,w)\in \cS_t}w\cdot \ell(h(\x),y)$.
\ENDFOR
\RETURN $h_T$
\end{algorithmic}

\end{algorithm}

In order to prove Theorem \ref{th:genbound}, we first present the following Lemma.
\begin{lemma}\label{lemma:genbound}
Define the weighted empirical loss as $L_{t-1}(f)=\frac{1}{t-1}\left(\sum_{i=1}^{t-1} \frac{Q_i}{p_i}\ell(f(\x_i),y_i) \right).$ For all $\delta>0$ with probability at least $1-\delta$, for all $t\in\bc{1,\ldots,T}$, and for all $f, g\in \cH_{t-1}$, we have
\begin{equation*}
|L_{t-1}(f) -L_{t-1}(g)-L(f)+L(g)| \leq \Delta_{t-1}.
\end{equation*}
\end{lemma}
\begin{proof}[Proof of Lemma \ref{lemma:genbound}]
Fix any $t\in\bc{1,\ldots,T}$. For any $f, g \in \cH_{t-1}$, define
$$Z_i= \frac{Q_i}{p_i}\left(\ell(f(\x_i),y_i)-\ell(g(\x_i),y_i)\right)-\left(L(f)-L(g)\right)$$
for $i\in\bc{1,\ldots,t-1}.$
The sequence $Z_i$ is a martingale difference since $$\E\bs{Z_i| Z_1, \ldots Z_{i-1}}=\E\bs{\frac{Q_i}{p_i}(\ell(f(\x_i),y_i)-\ell(g(\x_i),y_i)-(L(f)-L(g))| Z_1, \ldots Z_{i-1}}=0.$$
Due to the fact that $$|Z_i| \leq \frac{1}{p_i}|\ell(f(\x_i),y_i)- \ell(g(\x_i),y_i)| + |L(f) -L(g)|\leq 2$$
as $p_i \geq | \ell(f(\x_i),y_i) - \ell(g(\x_i),y_i) |$ for $f, g\in \cH_{t-1}$.
Thus, we can apply Azuma's inequality,
\begin{align*}
&\Pr[|L_{t-1}(f) - L_{t-1}(g)-L(f)+L(g)|\geq \Delta_{t-1}] \\
&  \leq 2 \exp(-(t-1) \Delta_{t-1}^2 / 8)= \frac{\delta}{T(T+1)|H|^2},
\end{align*}
where we used the fact that $(\x_t,y_t)$ is i.i.d.
Since $\cH_{t-1}$ is a random subset of $\cH$, we take the union bound over $f,g\in \cH$ and $t-1$. Then we take another union bound over $T$ to finish the proof.
\end{proof}

Next we provide the proof of Theorem \ref{th:genbound}.
\iwebsgenbound*
\begin{proof}[Proof of Theorem \ref{th:genbound}]
We show that $h^*\in \cH_t$ by induction. Base case holds as $h^*\in \cH_1=\cH$. Now assuming that $h^*\in \cH_{t-1}$ holds, we show that $h^*\in \cH_t$. Let $h'=\argmin_{f\in \cH_{t-1}} L_{t-1}(f)  $.
By Lemma~\ref{lemma:genbound}, $$L_{t-1}(h^*)-L_{t-1}(h^{'})\leq L(h^*) - L(h{'}) + \Delta_{t-1}\leq \Delta_{t-1}$$ since $L(h^*) - L(h')\leq 0$ by definition of $h^*$. Thus, $L_{t-1}(h^*) \leq L_{t-1}(h^{'} )+ \Delta_{t-1}$ which means that $h^*\in \cH_t$ by definition of $\cH_t$.

Since $\cH_t\subseteq H_{t-1}$, Lemma~\ref{lemma:genbound} implies that for any $f,g\in \cH_t$,
\begin{align*}
& L(f)-L(g) \leq L_{t-1}(f)- L_{t-1}(g) +\Delta_{t-1} \\
&\leq L_{t-1}(h')+ \Delta_{t-1}-L_{t-1}(h')+ \Delta_{t-1} \leq 2 \Delta_{t-1}.
\end{align*}
Noting that $h^*,h_t\in \cH_t$ completes the proof.
\end{proof}

\subsection{Proofs for statements in Section \ref{sec:samplerate}}

\coresetiwalsone*
\begin{proof}[Proof of Theorem \ref{th:coreset_iwals1}]
By Theorem~\ref{th:genbound}, $\cH_t \subset \{h\in \cH: L(h) \leq L(h^*) + 2\Delta_{t-1} \}$. Using this fact and that $h^*$ is the best in class, it holds that
\begin{align*}
\rho_\IWALSone(h,h^*) = \E_{(\x,y)\in \cD} | \ell(h(\x),y) - \ell(h^*(\x),y) | \leq L(h) + L(h^*) \leq 2L(h^*) + 2\Delta_{t-1} .
\end{align*}
Thus, letting $r= 2L(h^*) + 2\Delta_{t-1}$, it holds that $\cH_t\subset \cB_\IWALSone(h^*,r)$. Then,
\begin{align*}
\E_{(\x_t,y_t)\in\cD}[p_t|\cF_{t-1}] &=\E_{(\x_t,y_t)\in\cD}[\sup_{f,g\in \cH_{t}} |\ell(f(\x_t),y_t)-\ell(g(\x_t),y_t)||\cF_{t-1}]\\
& \leq 2\E_{(\x_t,y_t)\in\cD}[\sup_{h\in \cH_{t}} |\ell(h(\x_t),y_t)-\ell(h^*(\x_t),y_t)||\cF_{t-1}]  \\
& \leq 2\E_{(\x_t,y_t)\in\cD}[\sup_{h\in \cB_\IWALSone(h^*,r)} |\ell(h(\x_t),y_t)-\ell(h^*(\x_t),y_t)|]  \\
& \leq 2\theta_\IWALSone r= 4 \theta_\IWALSone (L(h^*) + \Delta_{t-1})
\end{align*}

By summing the above over $t\in[T]$, the sample complexity bound of \IWEBSV\ is then given by $\cO\br{\theta_\IWALSone L(h^*) T + \theta_\IWALSone \sqrt{T}}$, which completes the proof.
\end{proof}

\coresetiwalsonecomp*
\begin{proof}[Proof of Theorem \ref{th:coresetiwals1_comp}]
We first prove
\begin{align}\label{eq:ball_iwals_al}
\cB_\IWAL(h^*,r)\subseteq \cB_\IWALSone(h^*,r)\subseteq \cB_\IWAL(h^*,K_\ell r).
\end{align}
The left hand side follows directly from the definition. For the right hand side, assume $h\in\cB_\IWALSone(h^*,r)$, then by the definition of $K_\ell$, we have
\begin{align*}
r&\geq\rho_\IWALSone(h,h^*)=\E_{(\x,y)\in\cD}\bs{|\ell(h(\x),y)-\ell(h^*(\x),y)|}
\geq \frac{1}{K_{\ell}}\E_{\x\in\cX}\sup_{y\in\cY}|\ell(h(\x),y)-\ell(h^*(\x),y)|=\frac{1}{K_{\ell}}\rho_\IWAL(h,h^*)
\end{align*}
where the second inequality comes from
\begin{align*}
K_\ell&=\sup_{z,z'\in\cZ}\frac{\max_{y\in\cY}|\ell(z,y)-\ell(z',y)|}{\min_{y\in\cY}|\ell(z,y)-\ell(z',y)|} \\ & \geq \sup_{\x\sim\cX}\frac{\max_{y\in\cY}|\ell(h(\x),y)-\ell(h^*(\x),y)|}{\min_{y\in\cY}|\ell(h(\x),y)-\ell(h^*(\x),y)|}\\
&\geq\E_{\x\sim\cX}\frac{\max_{y\in\cY}|\ell(h(\x),y)-\ell(h^*(\x),y)|}{\min_{y\in\cY}|\ell(h(\x),y)-\ell(h^*(\x),y)|}\\
&\geq \frac{\E_{\x\sim\cX}\max_{y\in\cY}|\ell(h(\x),y)-\ell(h^*(\x),y)|}{\E_{\x\sim\cX}\min_{y\in\cY}|\ell(h(\x),y)-\ell(h^*(\x),y)|}\\
&\geq \frac{\E_{\x\sim\cX}\max_{y\in\cY}|\ell(h(\x),y)-\ell(h^*(\x),y)|}{\E_{(\x,y)\sim\cD}|\ell(h(\x),y)-\ell(h^*(\x),y)|}
\end{align*}
which follows by basic properties of expectations.
As a result, we have
\begin{align*}
\theta_\IWALSone &= \sup_{r\geq 0}\frac{\E_{(\x,y)\sim\cD}\left[\max_{h\in \cB_\IWALSone(h^*,r)}| \ell(h(x),y) - \ell(h^*(x),y)|\right]}{r} \tag{By definition of $\theta_\IWALSone$}\\
&\leq \sup_{r\geq 0}\frac{\E_{(\x,y)\sim\cD}\left[\max_{h\in \cB_\IWAL(h^*,K_{\ell} r)}| \ell(h(x),y) - \ell(h^*(x),y)|\right]}{r} \tag{Use equation~\eqref{eq:ball_iwals_al}}\\
&= K_{\ell} \sup_{r\geq 0}\frac{\E_{(\x,y)\sim\cD}\left[\max_{h\in \cB_\IWAL(h^*,r)}| \ell(h(x),y) - \ell(h^*(x),y)|\right]}{r}\tag{Redefine $K_\ell r$ as $r$}\\
&\leq K_{\ell}\sup_{r\geq 0}\frac{\E_{\x\in\cX}\left[\max_{h\in \cB_\IWAL(h^*,r)}\sup_{y\in\cY}| \ell(h(x),y) - \ell(h^*(x),y)|\right]}{r}\\
&=K_{\ell} \theta_\IWAL.
\end{align*}
which concludes the proof.
\end{proof}

\subsection{The Enhanced-\IWEBSV\ Algorithm}\label{sec:enhanced}

Here we analyze a modified algorithm, which we called the Enhanced-\IWEBSV\ ,  that uses new slack term defined as
\begin{align*}
\Delta^{\textrm{EIWeS}}_t =
\frac{2}{t} \Big(\sqrt{\sum_{t=1}^T p_t }+ 6\sqrt{\log\br{(3+t)t^2/\delta}}\Big) \cdot \sqrt{\log\br{8T^2|\cH|^2\log (T)/\delta}}
\end{align*}
to define the version space in the \IWEBSV\ algorithm. The theorem below proves that the Enhanced-\IWEBSV\ admits improved sampling rate bound that is smaller than the bound in Theorem~\ref{th:coreset_iwals1}  since the square root term also scales with $L(h^*)$. In particular, under realizable setting where $L(h^*)=0$, the sampling rate bound is $\cO\br{\log^3\br{T}}$, which depends poly-logarithmic in $T$ and is smaller than the $\cO\br{\sqrt{T\log\br{T}}}$ bound from Theorem~\ref{th:coreset_iwals1}.
\begin{restatable}{theorem}{eiwebsone}\label{th:eiwebs1}
For all $\delta>0$, for all $T\geq 3$, with probably at least $1-\delta$,   the expected sampling rate of the Enhanced-\IWEBSV\ algorithm is:
\begin{align*}
\sum_{t=1}^T\E_{(\x_t,y_t)\sim\cD} [p_t |\cF_{t-1}] \leq \theta_\IWALSone\br{L(h^*)T+\cO\br{\sqrt{L(h^*)T\log(T/\delta)}}}+\cO\br{\log^3(T/\delta)}.
\end{align*}
\end{restatable}
\begin{proof}[Proof of Theorem~\ref{th:eiwebs1}]
The proof follows from that of Lemma 6 in \cite{cortesdesalvogentilmohrizhang2019}. From Theorem~\ref{th:coreset_iwals1}, for $t\geq 3$,
\begin{align*}
\sum_{t=1}^T\E_{(\x_t,y_t)\sim\cD} \bs{p_t\big |\cF_{t-1}}
=4\theta_\IWALSone\br{L(h^*)+\Delta_{t-1}}.
\end{align*}
Plugging in the expression of $\Delta^{\textrm{EIWeS}}_{t-1}$ into $\Delta_{t-1}$, and applying a concentration inequality to relate
$\sum_{t=1}^T p_t$ to $\sum_{t=1}^T  \E_{(\x_t,y_t)\sim\cD}  \bs{p_t\big |\cF_{t-1}}$,
we end up with a recursion on $\E_{(\x_t,y_t)\sim\cD} \bs{ p_t | \cF_{t-1}}$:
\begin{align*}
\E_{(\x_t,y_t)\sim\cD}\bs{p_t\big |\cF_{t-1}}
\leq4\theta_\IWALSone L(h^*)+\frac{4\theta_\IWALSone c_1}{t-1}\sqrt{\sum_{s=1}^{t-1}\E_{(\x_t,y_t)\sim\cD}\bs{p_s|\cF_{s-1}}}+c_2\br{\frac{\log\br{(t-1)|\cH|/\delta}}{t-1}}.
\end{align*}
where $c_1=2\sqrt{\log\br{\frac{8T^2|\cH|^2\log\br{T}}{\delta}}}$, and $c_2$ is a constant. Then we show for all $t\geq 3$, for constant $c_3=\cO\br{\sqrt{\log\br{T|\cH|/\delta}}}$ and $c_4=\cO\br{\log^2\br{T|\cH|/\delta}}$, we have
\begin{align*}
\E_{(\x_t,y_t)\sim\cD}\bs{p_t\big |\cF_{t-1}}
\leq 4\theta_\IWALSone L(h^*)+c_3\sqrt{\frac{L(h^*)}{t-1}}+\frac{c_4}{t-1}
\end{align*}
The above can be proved by induction as in \cite{cortesdesalvogentilmohrizhang2019}. By summing the above over $t\in[T]$ gives us the result. We absorb the dependency on $|\cH|$ inside the big-O notation.

\end{proof}

\end{document}